\documentclass{article}

\usepackage{arxiv}

\usepackage[utf8]{inputenc} 
\usepackage[T1]{fontenc}    
\usepackage{url}            
\usepackage{booktabs}       
\usepackage{amsfonts}       
\usepackage{nicefrac}       
\usepackage{microtype}      
\usepackage{lipsum}
\usepackage{amsmath}
\usepackage{amsthm}
\usepackage{amssymb}
\usepackage{bigints}
\usepackage{bm}
\usepackage{diagbox}
\usepackage{caption}
\usepackage{subcaption}
\usepackage{graphicx}
\usepackage[ruled,vlined]{algorithm2e}
\include{pythonlisting}
\usepackage{xcolor}
\usepackage{mathtools}
\usepackage{enumitem}

\usepackage{hyperref}
\hypersetup{
    colorlinks=true,
    citecolor=black,
    linkcolor=black,
    filecolor=black,
    urlcolor=black,
}

\newcommand{\R}{\mathbb{R}}

\theoremstyle{definition}

\numberwithin{equation}{section}

\theoremstyle{plain}

\newtheorem{theorem}{Theorem}[section]
\newtheorem{definition}[theorem]{Definition}

\newtheorem{lemma}[theorem]{Lemma}

\newtheorem{remark}[theorem]{Remark}

\title{When and why PINNs fail to train: A neural tangent kernel perspective}

\author{
  Sifan Wang \\
  Graduate Group in Applied Mathematics \\
  and Computational Science \\
  University of Pennsylvania\\
  Philadelphia, PA 19104 \\
  \texttt{sifanw@sas.upenn.edu} \\
   \And
     Xinling Yu \\
  Graduate Group in Applied Mathematics \\
  and Computational Science \\
  University of Pennsylvania\\
  Philadelphia, PA 19104 \\
  \texttt{xlyu@sas.upenn.edu} \\
   \And
  Paris Perdikaris \\
  Department of Mechanichal Engineering \\
  and Applied Mechanics\\
  University of Pennsylvania\\
  Philadelphia, PA 19104 \\
  \texttt{pgp@seas.upenn.edu} \\
}

\begin{document}
\maketitle

\begin{abstract}
Physics-informed neural networks (PINNs) have lately received great attention thanks to their flexibility in tackling a wide range of forward and inverse problems involving partial differential equations. However, despite their noticeable empirical success, little is known about how such constrained neural networks behave during their training via gradient descent. More importantly, even less is known about why such models sometimes fail to train at all.
In this work, we aim to investigate these questions through the lens of the Neural Tangent Kernel (NTK); a kernel that captures the behavior of fully-connected neural networks in the infinite width limit during training via gradient descent. Specifically, we derive the NTK
of PINNs and prove that, under appropriate conditions, it converges to a deterministic kernel that stays constant during training in the infinite-width limit. This allows us to analyze the training dynamics of PINNs through the lens of their limiting NTK and find a remarkable discrepancy in the convergence rate of the different loss components contributing to the total training error. To address this fundamental pathology, we propose a novel gradient descent algorithm that utilizes the eigenvalues of the NTK to adaptively calibrate the convergence rate of the total training error. Finally, we perform a series of numerical experiments to verify the correctness of our theory and the practical effectiveness of the proposed algorithms. The data and code accompanying this manuscript are publicly available at  \url{https://github.com/PredictiveIntelligenceLab/PINNsNTK}.
\end{abstract}


\section{Introduction}

Thanks to  the approximation capabilities of neural networks, physics-informed neural networks (PINNs) have already led to a series of remarkable results across a range of problems in computational science and engineering,  including   fluids mechanics \cite{raissi2020hidden,sun2020surrogate,raissi2019deep2,jin2020nsfnets}, bio-engineering    \cite{sahli2020physics, kissas2020machine}, meta-material design \cite{fang2019deep,liu2019multi,chen2020physics},  free boundary problems \cite{wang2020deep},
Bayesian networks and uncertainty quantification \cite{yang2019adversarial,zhu2019physics, yang2018physics,sun2020physics, yang2020b}, high-dimensional PDEs \cite{sirignano2018dgm, han2018solving},
stochastic differential equations \cite{zhang2020learning}, fractional differential equations \cite{pang2019fpinns, pang2020npinns}, and beyond \cite{tartakovsky2018learning,lu2019deeponet,tartakovsky2020physics, shin2020convergence}. However, PINNs using fully connected architectures often fail to achieve stable training and produce accurate predictions, especially when the underlying PDE solutions contain high-frequencies or multi-scale features \cite{tchelepi2020limitations,zhu2019physics,raissi2018deep}. Recent work by Wang {\em et. al.} \cite{wang2020understanding} attributed this pathological behavior to multi-scale interactions between different terms in the PINNs loss function, ultimately leading to stiffness in the gradient flow dynamics, which, consequently, introduces stringent stability requirements on the learning rate. To mitigate this pathology, Wang {\em et. al.} \cite{wang2020understanding} proposed an empirical learning-rate annealing scheme that utilizes the back-propagated gradient statistics during training to adaptively assign importance weights to different terms in a PINNs loss function, with the goal of balancing the magnitudes of the back-propagated gradients. Although this approach was demonstrated to produce significant and consistent improvements in the trainability and accuracy of PINNs, the fundamental reasons behind the practical difficulties of training fully-connected PINNs still remain unclear \cite{tchelepi2020limitations}.

Parallel to the development of PINNs, recent investigations have shed light into the representation shortcomings and training deficiencies of fully-connected neural networks. Specifically, it has been shown that conventional fully-connected architectures -- such as the ones typically used in PINNs -- suffer from ``spectral bias" and are incapable of learning functions with high frequencies, both in theory and in practice \cite{rahaman2019spectral,cao2019towards,tancik2020fourier,basri2020frequency}. These observations are rigorously grounded by the newly developed neural tangent kernel theory \cite{jacot2018neural,yang2019scaling} that, by exploring the connection between deep neural networks and kernel regression methods, elucidates the training dynamics of deep learning models. Specifically, the original work of Jacot {\em et. al.} \cite{jacot2018neural} proved that, at initialization, fully-connected networks are equivalent to Gaussian
processes in the infinite-width limit \cite{matthews2018gaussian, lee2017deep,mackay1998introduction}, while the evolution of a infinite-width network during training can also be described by another kernel, the so-called Neural Tangent Kernel (NTK) \cite{jacot2018neural}. Remarkably, this function-space viewpoint allows us then to rigorously analyze the training convergence of deep neural networks by examining the spectral properties of their limiting NTK \cite{jacot2018neural,yang2019scaling,basri2020frequency}.

Drawing motivation from the aforementioned developments, this work sets sail into investigating the training dynamics of PINNs. To this end, we rigorously study fully-connected PINNs models through the lens of their neural tangent kernel, and produce novel insights into when and why such models can be effectively trained, or not. Specifically, our main contributions can be summarized into the following points:
\begin{itemize}
    \item We prove that fully-connected PINNs converge to Gaussian processes at the infinite width limit for linear PDEs.
    \item We derive the neural tangent kernel (NTK) of PINNs and prove that, under suitable assumptions, it converges to a deterministic kernel and remains constant during training via gradient descent with an infinitesimally small learning rate.
    \item We show how the convergence rate of the total training error a PINNs model can be analyzed in terms of the spectrum of its NTK at initialization.
    \item We show that fully-connected PINNs not only suffer from spectral bias, but also from a remarkable discrepancy of convergence rate in the different components of their loss function.
    \item We propose a novel adaptive training strategy for resolving this pathological convergence behavior, and significantly enhance the trainablity and predictive accuracy of PINNs.
\end{itemize}
Taken together, these developments provide a novel path into analyzing the convergence of PINNs, and enable the design of novel training algorithms that can significantly improve their trainability, accuracy and robustness.

This paper are organized as follows. In section \ref{sec: infinite_wide_NN}, we present a brief overview of fully-connected neural networks and their behavior in the infinite-width limit following the original formulation of Jacot {\em et. al.} \cite{jacot2018neural}. Next, we derive the NTK of PINNs in a general setting and prove that, under suitable assumptions, it converges to a deterministic kernel and remains constant during training via gradient descent with an infinitesimally small learning rate, see section  \ref{sec: PINNs_formulation}, \ref{sec: PINN_dynamics}.
Furthermore, in section \ref{sec: spectral_bias} we analyze the training dynamics of PINNs, demonstrate that PINNs models suffer from spectral bias, and then propose a novel algorithm to improve PINNs' performance in practice. Finally we carry out a series of numerical experiments to verify the developed NTK theory and validate the effectiveness of the proposed algorithm.

\section{Infinitely Wide  Neural Networks}
\label{sec: infinite_wide_NN}
In this section, we revisit the definition of fully-connected neural networks and investigate their behavior under the infinite-width limit. Let us start by formally defining the forward pass of a scalar valued fully-connected network with $L$ hidden layers, with the input and output dimensions denoted as $d_0 = d$, and $d_{L+1} = 1$, respectively. For inputs $\bm{x} \in \R^d$ we also denote the input layer of the network as $\bm{f}^{(0)}(\bm{x}) = \bm{x}$ for convenience. Then a fully-connected neural network with $L$ hidden layers is defined recursively as
\begin{align}
    &\bm{f}^{(h)}(\bm{x}) =  \frac{1}{\sqrt{d_h}} \bm{W}^{(h)} \cdot \bm{g}^{(h)}  + \bm{b}^{(h)}  \in \R^{d_{h+1}}, \\
    &\bm{g}^{(h)}(\bm{x}) =
    \sigma (\bm{W}^{(h-1)}  \bm{f}^{(h-1)}(\bm{x})  + \bm{b}^{(h-1)}),
\end{align}
for $h = 1, \dots, L$, where $\bm{W}^{(h)} \in \R^{d_{h+1} \times d_{h}}$ are  weight matrices and $\bm{b}^{(h)} \in \R^{d_{h+1}}$ are bias vectors in the $h$-th hidden layer, and $\sigma: \R \rightarrow \R$ is a coordinate-wise smooth activation function. The final output of the neural network is given by
\begin{align}
    \label{eq: NTK_param}
    f(\bm{x}, \bm{\theta}) &= \bm{f}^{(L)}(\bm{x}) =  \frac{1}{\sqrt{d_L}}\bm{W}^{(L)} \cdot \bm{g}^{(L)}(\bm{x}) + \bm{b}^{(L)},
\end{align}
where $\bm{W}^{(L)} \in \R^{1 \times d_{L}}$ and $\bm{b}^{(L)} \in \R$ are the weight and bias parameters of the last layer. Here, $\bm{\theta} = \{\bm{W}^{(0)}, \bm{b}^{(0)}, \dots, \bm{W}^{(L)}, \bm{b}^{(L)}\}$ represents all parameters of the neural network. We remark that such a parameterization  is known as the ``NTK parameterization" following the original work of Jacot {\em et. al.} \cite{jacot2018neural}.

We initialize all the weights and biases to be independent and identically distributed (i.i.d.) as standard normal distribution $\mathcal{N}(0, 1)$ random variables, and consider the sequential limit of hidden widths $d_1, d_2, \dots, d_L \rightarrow \infty$.  As described in \cite{jacot2018neural, lee2017deep, matthews2018gaussian},  all coordinates of $\bm{f}^{(h)}$ at each hidden layer asymptotically converge to an i.i.d centered Gaussian process with covariance $\Sigma^{h-1}: \R^{d_{h-1}} \times \R^{d_{h-1}}\rightarrow \R$ defined recursively as
\begin{align}
    \label{eq: NN_Gaussian_cov}
    \begin{split}
        &\Sigma^{(0)}(\bm{x}, \bm{x}') = \bm{x}^T \bm{x}' + 1, \\
     &\bm{\Lambda}^{(h)}(\bm{x}, \bm{x}') =
    \begin{pmatrix}
    \Sigma^{(h-1)}\left(\bm{x}, \bm{x}\right) & \Sigma^{(h-1)}\left(\bm{x}, \bm{x}^{\prime}\right) \\
    \Sigma^{(h-1)}\left(\bm{x}', \bm{x}\right) & \Sigma^{(h-1)}\left(\bm{x}', \bm{x}'\right)
    \end{pmatrix} \in \R^{2 \times 2 }, \\
    & \Sigma^{(h)}(\bm{x}, \bm{x}') = \mathop{\mathbb{E}}\limits_{(u, v)\sim \mathcal{N}(0, \Lambda^{(1)})} [\sigma(u)\sigma(v)] + 1,
    \end{split}
\end{align}
for $h= 1, 2, \dots, L$.

To introduce the neural tangent kernel (NTK), we also need to define
\begin{align}
\label{eq: dot_Sigma}
    \dot{\Sigma}^{(h)}\left(\boldsymbol{x}, \boldsymbol{x}^{\prime}\right)=\underset{(u, v) \sim \mathcal{N}\left(0, \Lambda^{(h)}\right)}{\mathbb{E}}[\dot{\sigma}(u) \dot{\sigma}(v)]
\end{align}
where $\dot{\sigma}$ denotes the derivative of the activation function $\sigma$.

Following the derivation of \cite{jacot2018neural, arora2019exact}, the neural tangent kernel can be  generally defined at any time $t$, as the neural network parameters $\theta(t)$ are changing during model training by gradient descent. This definition takes the form
\begin{align}
    Ker_t(\bm{x}, \bm{x}') = \left\langle\frac{\partial f(\bm{x}, \boldsymbol{\theta}(t))}{\partial \boldsymbol{\theta}}, \frac{\partial f\left(\bm{x}' , \boldsymbol{\theta}(t) \right)}{\partial \boldsymbol{\theta}}\right\rangle,
\end{align}
and this kernel converges in probability to a deterministic kernel $\Theta^{(L)}(\bm{x}, \bm{x}')$  at random initialization as the width of hidden layers goes to infinity \cite{jacot2018neural}. Specifically,
\begin{align}
    &\lim _{d_{L} \rightarrow \infty} \cdots \lim _{d_{1} \rightarrow \infty}  Ker_0(\bm{x}, \bm{x}') =  \lim _{d_{L} \rightarrow \infty} \cdots \lim _{d_{1} \rightarrow \infty}  \left\langle\frac{\partial f(\bm{x}, \boldsymbol{\theta}(0))}{\partial \boldsymbol{\theta}}, \frac{\partial f\left(\bm{x}' , \boldsymbol{\theta}(0) \right)}{\partial \boldsymbol{\theta}}\right\rangle = \Theta^{(L)}(\bm{x}, \bm{x}').
\end{align}
Here $\Theta^{(L)}(\bm{x}, \bm{x}')$ is recursively defined by
\begin{align}
    \label{eq: NTK}
    &\Theta^{(L)}\left(\boldsymbol{x}, \boldsymbol{x}^{\prime}\right)=\sum_{h=1}^{L+1}\left(\Sigma^{(h-1)}\left(\boldsymbol{x}, \boldsymbol{x}^{\prime}\right) \cdot \prod_{h^{\prime}=h}^{L+1} \dot{\Sigma}^{\left(h^{\prime}\right)}\left(\boldsymbol{x}, \boldsymbol{x}^{\prime}\right)\right)
\end{align}
where $\dot{\Sigma}^{(L+1)}(\bm{x}, \bm{x}') = 1$ for convenience.
Moreover, Jacot {\em et. al.} \cite{jacot2018neural}  proved that, under some suitable conditions and training time $T$ fixed, $Ker_t$ converges to $ \Theta^{(L)}$ for all $0 \leq t \leq T$, as the width goes to infinity. As a consequence, a properly randomly initialized and sufficiently wide deep neural network trained by gradient descent is equivalent to a kernel regression with a deterministic kernel.

\section{Physics-informed Neural Networks (PINNs)}
\label{sec: PINNs_formulation}
In this section, we study physics-informed neural networks (PINNs) and their corresponding neural tangent kernels. To this end, we consider the following well-posed partial differential equation (PDE) defined on a bounded domain $\Omega \subset \R^d$
\begin{align}
    \label{eq: PDE}
    &\mathcal{L}[u](\bm{x}) = f(\bm{x}),  \quad x \in \Omega  \\
    \label{eq: PDE_BC}
    & u(\bm{x}) = g(\bm{x}),  \quad x \in \partial \Omega
\end{align}
where $\mathcal{L}$ is a differential operator and $u(\bm{x}): \overline{\Omega} \rightarrow \R$ is the unknown solution with
$\bm{x} = (x_1, x_2, \cdots, x_d)$. Here we remark that for time-dependent problems, we consider time $t$ as an additional coordinate in $\bm{x}$ and $\Omega$ denotes the spatio-temporal domain. Then, the initial condition can be simply treated as a special type of Dirichlet boundary condition and included in equation \ref{eq: PDE_BC}.

Following the original work of Raissi {\em et. al.} \cite{raissi2019physics}, we assume that the latent solution $u(\bm{x})$ can be approximated by a deep neural network $u(\bm{x}, \bm{\theta})$ with parameters $\bm{\theta}$, where $\bm{\theta}$ is a collection of all the parameters in the network.
We can then define the PDE residual $r(\bm{x}, \bm{\theta})$ as
\begin{align}
r(\bm{x}, \bm{\theta}) := \mathcal{L}u(\bm{x}, \bm{\theta}) - f(\bm{x}).
\end{align}
Note that the parameters of $u(\bm{x}, \bm{\theta})$ can be ``learned" by minimizing the following composite loss function
\begin{align}
    \label{eq: PINN_loss}
    \mathcal{L}(\bm{\theta}) = \mathcal{L}_b(\bm{\theta}) + \mathcal{L}_r(\bm{\theta}),
\end{align}
where
\begin{align}
    &\mathcal{L}_b(\bm{\theta}) = \frac{1}{2} \sum_{i=1}^{N_b}  |u(\bm{x}_b^i, \bm{\theta}) - g(\bm{x}_b^i)   |^2 \\
    & \mathcal{L}_r(\bm{\theta}) =  \frac{1}{2} \sum_{i=1}^{N_r} |r(\bm{x}_r^i, \bm{\theta})|^2.
\end{align}
Here $N_b$ and $N_r$ denote the batch sizes for the training data $\{ {\bm{x}_b^i, g(\bm{x}_b^i)} \}_{i=1}^{N_b}$ and  $\{ {\bm{x}_r^i, f(\bm{x}_b^i)} \}_{i=1}^{N_r}$ respectively, which can be randomly sampled at each iteration of a gradient descent algorithm.

\subsection{Neural tangent kernel theory for PINNs}
In this section we derive the neural tangent kernel of a physics-informed neural network.
To this end, consider minimizing the loss function \ref{eq: PINN_loss} by gradient descent with an infinitesimally small learning rate, yielding the continuous-time gradient flow system
\begin{align}
    \label{eq: gradient_flow}
    \frac{d \bm{\theta}}{dt} = - \nabla \mathcal{L}(\bm{\theta}),
\end{align}
and let $u(t) = u(\bm{x}_b, \bm{\theta}(t)) = \{u(\bm{x}_b^i, \bm{\theta}(t))\}_{i=1}^{N_b}$ and $\mathcal{L}u(t) = \mathcal{L}u(\bm{x}_r, \bm{\theta}(t))= \{ \mathcal{L}u(\bm{x}_r, \bm{\theta}(t))  \}_{i=1}^{N_r}$. Then the following lemma characterizes how $u(t)$ and $\mathcal{L}u(t)$ evolve during training by gradient descent.

\begin{lemma}
\label{lemma: PINN_ode}
Given the data points  $\{\bm{x}_b^i, g(\bm{x}_b^i)\}_{i=1}^{N_b}, \{\bm{x}_r^i, f(\bm{x}_r^i)\}_{i=1}^{N_r}$ and the gradient flow \ref{eq: gradient_flow},  $u(t)$ and $\mathcal{L}u(t)$ obey the following evolution
\begin{align}
\label{eq: PINN_ode}
    \begin{bmatrix}
    \frac{d u(\bm{x}_b, {\bm \theta}(t))}{dt}\\
    \frac{d \mathcal{L}u(\bm{x}_r, {\bm \theta}(t))}{dt}
    \end{bmatrix}
    =
       - \begin{bmatrix}
     \bm{K}_{uu}(t) & \bm{K}_{ur}(t) \\
     \bm{K}_{ru}(t) & \bm{K}_{rr}(t)
    \end{bmatrix}
    \cdot
       \begin{bmatrix}
    u(\bm{x}_b, {\bm \theta}(t)) - g(\bm{x}_b) \\
    \mathcal{L}u(\bm{x}_r, {\bm \theta}(t)) - f(\bm{x}_r)
    \end{bmatrix},
\end{align}
where $\bm{K}_{ru}(t) = \bm{K}_{ur}^T(t)$ and
$\bm{K}_{uu}(t) \in \R^{N_b \times N_b}, \bm{K}_{ur}(t) \in \R^{N_b \times N_r}, and \bm{K}_{rr}(t) \in \R^{N_r \times N_r}$ whose $(i,j)$-th entry is given by
\begin{align}
\label{eq: NTK_PINN}
  \begin{split}
       & (\bm{K}_{uu})_{ij}(t) =  \Big\langle  \frac{d u({\bm x}_b^i, {\bm \theta}(t))}{d{\bm \theta}},  \frac{d u(\bm{x}_b^j, {\bm \theta}(t))}{d{\bm \theta}}    \Big\rangle \\
   & (\bm{K}_{ur})_{ij}(t)  =  \Big\langle  \frac{d u(\bm{x}_b^i, {\bm \theta}(t))}{d{\bm \theta}} , \frac{d \mathcal{L}u(\bm{x}_r^j, {\bm \theta}(t))}{d{\bm \theta}}  \Big\rangle  \\
   & (\bm{K}_{rr})_{ij}(t)  =  \Big\langle \frac{d \mathcal{L}(\bm{x}_r^i, {\bm \theta}(t))}{d{\bm \theta}},  \frac{d \mathcal{L}(\bm{x}_r^j, {\bm \theta}(t))}{d{\bm \theta}}   \Big\rangle
  \end{split}
\end{align}

\end{lemma}

\begin{proof}
The proof of lemma \ref{lemma: PINN_ode} is given in Appendix \ref{sec: proof_lemma_PINN_ode}.
\end{proof}

\begin{remark} $\langle \cdot, \cdot \rangle$ here denotes the inner product over all neural network parameters in $\bm{\theta}$. For example,
\begin{align*}
    (\bm{K}_{uu})_{ij}(t) = \sum_{\theta \in \bm{\theta}} \frac{d u({\bm x}_b^i, {\bm \theta}(t))}{d{ \theta}} \cdot \frac{d u(\bm{x}_b^j, {\bm \theta}(t))}{d{ \theta}}.
\end{align*}
\end{remark}

\begin{remark}
\label{remark: NTK_PSD}
We will denote the matrix $  \begin{bmatrix}
     \bm{K}_{uu}(t) & \bm{K}_{ur}(t) \\
     \bm{K}_{ru}(t) & \bm{K}_{rr}(t)
    \end{bmatrix}
     $ by $\bm{K}(t)$ in the following sections. It is easy to see that $\bm{K}_{uu}(t), \bm{K}_{rr}(t)$ and $\bm{K}(t)$ are all positive semi-definite matrices. Indeed, let $J_u(t)$ and $J_r(t)$ be the Jacobian matrices of $u(t)$ and $\mathcal{L}u(t)$ with respect to $\bm{\theta}$ respectively. Then, we can observe that
     \begin{align*}
         &\bm{K}_{uu}(t) = \bm{J}_u(t) \bm{J}_u^T(t), \quad \bm{K}_{rr}(t) = \bm{J}_r(t) \bm{J}_r^T(t), \quad \bm{K}(t) = \begin{bmatrix}
                        \bm{J}_u(t) \\
                        \bm{J}_r(t)
                        \end{bmatrix}
                         \begin{bmatrix}
                        \bm{J}_u^T(t) , \bm{J}_r^T(t)
                        \end{bmatrix}.
     \end{align*}
\end{remark}

\begin{remark}
\label{remark: NTK_general}
It is worth pointing out that equation \ref{eq: PINN_ode} holds for any differential operator $\mathcal{L}$ and any neural network architecture.
\end{remark}

The statement of Lemma \ref{lemma: PINN_ode} involves the matrix $\bm{K}(t)$, which we call the {\em neural tangent kernel of a physics-informed neural network (NTK of PINNs)}.  Recall that an infinitely wide neural network is a Gaussian process, and its NTK remains constant during training \cite{jacot2018neural}.
Now two natural questions arise: how does the PDE residual behave in the infinite width limit? Does the NTK of PINNs exhibit similar behavior as the standard NTK? If so, what is the expression of the corresponding kernel? In the next subsections, we will answer these questions and show that, in the infinite-width limit, the NTK of PINNs indeed converges to a deterministic kernel at initialization and then remains constant during training.

\section{Analyzing the training dynamics of PINNs through the lens of their NTK}
\label{sec: PINN_dynamics}
To simplify the proof and understand the key ideas clearly, we confine ourselves to a simple model problem using a fully-connected neural network with one hidden layer. To this end, we consider a one-dimensional Poisson equation as our model problem. Let $\Omega$ be a bounded open interval in $\R$. The partial differential equation is summarized as follows
\begin{align}
\label{eq: model problem}
    \begin{split}
        & u_{xx}(x) = f(x), \quad \forall x \in \Omega \\
    & u(x) = g(x), \quad x \in \partial \Omega.
    \end{split}
\end{align}
We proceed by approximating the solution $u(x)$ by a fully-connected neural network denoted by $u(x, \bm{\theta})$ with one hidden layer. Now we define the network explicitly:
\begin{align}
    u(x, \bm{\theta}) = \frac{1}{\sqrt{N}} \bm{W}^{(1)} \cdot \sigma(\bm{W}^{(0)}x + \bm{b}^{(0)}) + \bm{b}^{(1)},
\end{align}
where $\bm{W}^{(0)} \in \R^{N \times 1}, \bm{W}^{(1)} \in \R^{1 \times N}$ are weights,  $\bm{b}^{(0)} \in \R^N, \bm{b}^{(1)}\in \R^1 $ are biases  $\bm{\theta} = (\bm{W}^{(0)} , \bm{W}^{(1)}, \bm{b}^{(0)}, \bm{b}^{(1)})$ represents all parameters in the network , and $\sigma$ is a smooth activation function. Then it is straightforward to show that
\begin{align}
    \label{eq: u_xx}
    u_{xx}(x, \bm{\theta}) = \frac{1}{\sqrt{N}} \bm{W}^{(1)} \cdot \Big[  \Ddot{\sigma}(\bm{W}^{(0)}x + \bm{b}^{(0)} ) \odot \bm{W}^{(0)}  \odot \bm{W}^{(0)} \Big],
\end{align}
where $\odot$ denotes point-wise multiplication and $\Ddot{\sigma}$ denotes the second order derivative of the activation function $\sigma$.

We initialize all the weights and biases to be i.i.d. $\mathcal{N}(0, 1)$ random variables. Based on our presentation in section \ref{sec: infinite_wide_NN}, we already know that, in the infinite width limit, $u(x,\bm{\theta})$ is a centered Gaussian process with covariance matrix $\Sigma^{(1)}(x ,x')$ at initialization, which is defined in equation \ref{eq:  NN_Gaussian_cov}. The following theorem reveals that $u_{xx}(x, \bm{\theta})$ converges in distribution to another centered Gaussian process with a covariance $\Sigma_{xx}^{(1)}$ under the same limit.

\begin{theorem}
\label{theorem: u_xx_gaussian}
Assume that the activation function $\sigma$ is smooth and has a bounded second order derivative $\Ddot{\sigma}$. Then for a fully-connected neural network of one hidden layer at initialization,
\begin{align}
   &u(x, \bm{\theta}) \xrightarrow{\mathcal{D}} \mathcal{GP}(0, \Sigma^{(1)}(x,x')) \\
   &u_{xx}(x, \bm{\theta}) \xrightarrow{\mathcal{D}} \mathcal{GP}(0, \Sigma_{xx}^{(1)}(x,x')),
\end{align}
as $N \rightarrow \infty$,
where $\mathcal{D}$ means convergence in distribution and
\begin{align}
     \Sigma_{xx}^{(1)}(x, x') =  \mathop{\mathbb{E}}_{u,v \sim \mathcal{N}(0,1)} \Big[ u^4 \Ddot{\sigma}(ux + v) \Ddot{\sigma}(ux' + v)   \Big].
\end{align}
\end{theorem}

\begin{proof}
The proof can be found in Appendix \ref{sec: proof_u_xx_gaussian}.
\end{proof}

\begin{remark}
By induction, the proof of Theorem \ref{theorem: u_xx_gaussian} can be extended to differential operators of any order and fully-connected neural networks with multiple hidden layers. Observe that a linear combination of Gaussian processes is still a Gaussian process. Therefore, Theorem \ref{theorem: u_xx_gaussian} can be generalized to any linear partial differential operator under appropriate regularity conditions.
\end{remark}

As an immediate corollary, a sufficiently wide physics-informed neural network for model problem \ref{eq: model problem} induces a joint Gaussian process (GP) between the function values and the PDE residual at initialization, indicating a PINNs-GP correspondence for linear PDEs.

The next question we investigate is whether the NTK of PINNs behaves similarly as the NTK of standard neural networks. The next theorem proves that indeed the kernel $\bm{K}(0)$ converges in probability to a certain deterministic kernel matrix as the width of the network goes to infinity.

\begin{theorem}
\label{theorem: NTK_PINNs_init}
For a physics-informed network with one hidden layer at initialization, and in the limit as the layer's width $N \rightarrow \infty$, the NTK $\bm{K}(t)$ of the PINNs model defined in equation \ref{eq: NTK_PINN} converges in probability to a deterministic limiting kernel, i.e,
\begin{align}
    \bm{K}(0) =
    \begin{bmatrix}
     \bm{K}_{uu}(0) & \bm{K}_{ur}(0) \\
     \bm{K}_{ru}(0) & \bm{K}_{rr}(0)
    \end{bmatrix}
    \rightarrow
     \begin{bmatrix}
     \Theta_{uu}^{(1)} &   \Theta_{ur}^{(1)} \\
       \Theta_{ru}^{(1)} & \Theta_{rr}^{(1)}
    \end{bmatrix} := \bm{K}^*,
\end{align}
where the explicit expression of $\bm{K}^*$ is provided in appendix \ref{sec: proof_NTK_PINN_init}.
\end{theorem}

\begin{proof}
The proof can be found in Appendix \ref{sec:proof_NTK_PINNs_init}.
\end{proof}

Our second key result is that the NTK of PINNs stays asymptotically constant during training, i.e $\bm{K}(t) \approx \bm{K}(0)$ for all $t$.  To state and prove the theorem rigorously, we may assume that all parameters and the loss function do not blow up and are uniformly bounded during training. The first two assumptions are both reasonable and practical, otherwise one would obtain unstable and divergent training processes. In addition, the activation has to be $4$-th order smooth and all its derivatives are bounded. The last assumption is not a strong restriction since it is satisfied by most of the activation functions commonly used for PINNs such as sigmoids, hyperbolic tangents, sine functions, etc.
\begin{theorem}
\label{theorem: kernel_constant}
For the model problem \ref{eq: model problem} with a fully-connected neural network of one hidden layer, consider minimizing the loss function \ref{eq: PINN_loss} by gradient descent with an infinitesimally small learning rate. For any $T>0$ satisfying the following assumptions:
\begin{enumerate}[label=(\roman*),leftmargin=*]
    \item  there exists a constant $C >0 $ such that all parameters of the network is uniformly bounded for $t \in T$, i.e.
    \begin{align*}
        \sup_{t \in  [0,T]} \| \bm{\theta}(t)  \|_\infty   \leq C
    \end{align*}
    where $C$ does not depend on $N$.
    \item there exists a constant $C >0 $ such that
    \begin{align*}
        &\int_0^T \Big| \sum_{i = 1}^{N_b} \big(u(x_b^i,\bm{\theta}(\tau)) -g(x_b^i) \big)  \Big| d \tau \leq C \\
        &\int_0^T \Big| \sum_{i = 1}^{N_r} \big(u_{xx}(x_r^i,\bm{\theta}(\tau)) -f(x_r^i) \big)  \Big| d \tau \leq C
    \end{align*}
    \item the activation function $\sigma$ is smooth and  $|\sigma^{(k)}| \leq C$ for $k=0,1,2,3, 4$, where $\sigma^{(k)}$ denotes $k$-th order derivative of $\sigma$.
\end{enumerate}
Then we have
\begin{align}
    \lim_{N \rightarrow \infty} \sup_{t \in [0,T]} \| \bm{K}(t) - \bm{K}(0) \|_2  = 0,
\end{align}
where $\bm{K}(t)$ is the corresponding NTK of PINNs.
\end{theorem}

\begin{proof}
The proof can be found in Appendix \ref{sec: proof_kernel_constant}.
\end{proof}



Here we provide some intuition behind the proof. The crucial observation is that all parameters of the network change little during training (see Lemma \ref{lemma: weight_change_little} in the Appendix).
By intuition, for sufficient wide neural networks, any slight movement of weights would contribute to a non-negligible change in the network output. As a result, the gradients of the outputs $u(x, \bm{\theta})$ and $u_{xx}(x, \bm{\theta})$ with respect to parameters barely change (see Lemma \ref{lemma: output_grad_constant}), and, therefore, the kernel remains almost static during training.

Combining Theorem \ref{theorem: NTK_PINNs_init} and Theorem \ref{theorem: kernel_constant} we may conclude that, for the model problem of equation \ref{eq: model problem}, we have
\begin{align*}
    \bm{K}(t) \approx \bm{K}(0) \approx \bm{K}^*, \quad \forall t,
\end{align*}
from which (and equation \ref{eq: PINN_ode}) we immediately obtain
\begin{align}
    \label{eq: PINN_ode_approx}
     \begin{bmatrix}
    \frac{d u(\bm{x}_b, {\bm \theta}(t))}{dt}\\
    \frac{d u_{xx}(\bm{x}_r, {\bm \theta}(t))}{dt}
    \end{bmatrix}
    =
       - \bm{K}(t)
    \cdot
       \begin{bmatrix}
    u(\bm{x}_b, {\bm \theta}(t)) - g(\bm{x}_b) \\
    u_{xx}(\bm{x}_r, {\bm \theta}(t)) - f(\bm{x}_r)
    \end{bmatrix}
    \approx - \bm{K}^*   \begin{bmatrix}
    u(\bm{x}_b, {\bm \theta}(t)) - g(\bm{x}_b) \\
    u_{xx}(\bm{x}_r, {\bm \theta}(t)) - f(\bm{x}_r).
    \end{bmatrix}.
\end{align}
Note that if the matrix $\bm{K}^*$ is invertible, then
according to \cite{lee2019wide, tancik2020fourier}, the network’s outputs $u(x, \bm{\theta})$ and $u_{xx}(x, \bm{\theta})$ can be approximated for any arbitrary test data $\bm{x}_{test}$ after $t$ steps of gradient descent as
\begin{align}
      \begin{bmatrix}
    u(\bm{x}_{test}, {\bm \theta}(t)) \\
    u_{xx}(\bm{x}_{test}, {\bm \theta}(t))
    \end{bmatrix}
    \approx \bm{K}_{test}^* (\bm{K}^*)^{-1} \left(I - e^{- \bm{K}^*t }   \right)
    \cdot
     \begin{bmatrix}
     g(\bm{x}_b) \\
    f(\bm{x}_r)
    \end{bmatrix},
 \end{align}
 where  $\bm{K}_{test}$ is the NTK matrix between
all points in $\bm{x}_{test}$ and all training data. Letting $t \rightarrow \infty$, we obtain
\begin{align*}
      \begin{bmatrix}
    u(\bm{x}_{test}, {\bm \theta}(\infty)) \\
    u_{xx}(\bm{x}_{test}, {\bm \theta}(\infty))
    \end{bmatrix}
    \approx \bm{K}_{test}^* (\bm{K}^*)^{-1}
    \cdot
     \begin{bmatrix}
     g(\bm{x}_b) \\
    f(\bm{x}_r)
    \end{bmatrix}.
\end{align*}
This implies that, under the assumption that $\bm{K}^*$ is invertible, an infinitely wide physics-informed neural network for model problem \ref{eq: model problem} is also equivalent to a kernel regression. However, from the authors' experience, the NTK of PINNs is always degenerate (see Figures \ref{fig: eigval_change_width_500}, \ref{fig: eigval_change_3_hidden_layer}) which means that we may not be able to casually perform kernel regression predictions in practice.

\section{Spectral bias in physics-informed neural networks}
\label{sec: spectral_bias}

In this section, we will utilize the developed theory to investigate whether physics-informed neural networks are spectrally biased. The term ``spectral bias" \cite{rahaman2019spectral, xu2019frequency, basri2020frequency} refers to a well known pathology that prevents deep fully-connected networks from learning high-frequency functions.

Since the NTK of PINNs barely changes during training, we may rewrite  equation  \ref{eq: PINN_ode_approx} as
\begin{align}
     \begin{bmatrix}
    \frac{d u(\bm{x}_b, {\bm \theta}(t))}{dt}\\
    \frac{d u_{xx}(\bm{x}_r, {\bm \theta}(t))}{dt}
    \end{bmatrix}
    \approx - \bm{K}(0)  \begin{bmatrix}
    u(\bm{x}_b, {\bm \theta}(t)) - g(\bm{x}_b) \\
    u_{xx}(\bm{x}_r, {\bm \theta}(t)) - f(\bm{x}_r)
    \end{bmatrix},
\end{align}
which leads to
\begin{align}
     \begin{bmatrix}
    \frac{d u(\bm{x}_b, {\bm \theta}(t))}{dt}\\
    \frac{d u_{xx}(\bm{x}_r, {\bm \theta}(t))}{dt}
    \end{bmatrix}
    \approx \left( I - e^{- \bm{K}(0)t} \right)  \cdot
     \begin{bmatrix}
     g(\bm{x}_b) \\
    f(\bm{x}_r)
    \end{bmatrix}.
\end{align}
As mentioned in remark \ref{remark: NTK_PSD}, the NTK of PINNs is also positive semi-definite. So we can take its spectral decomposition $\bm{K}(0) = \bm{Q}^T\bm{\Lambda} \bm{Q}$, where $\bm{Q}$ is an orthogonal matrix and $\bm{\Lambda}$ is a diagonal matrix whose entries are the eigenvalues $\lambda_i \geq 0 $ of $\bm{K}(0)$. Consequently,
the training error is given by
\begin{align*}
     \begin{bmatrix}
    \frac{d u(\bm{x}_b, {\bm \theta}(t))}{dt}\\
    \frac{d u_{xx}(\bm{x}_r, {\bm \theta}(t))}{dt}
    \end{bmatrix} -
     \begin{bmatrix}
     g(\bm{x}_b) \\
    f(\bm{x}_r)
    \end{bmatrix}
    &\approx    \left( I - e^{- \bm{K}(0)t} \right)  \cdot
     \begin{bmatrix}
     g(\bm{x}_b) \\
    f(\bm{x}_r)
    \end{bmatrix}  -
    \begin{bmatrix}
     g(\bm{x}_b) \\
    f(\bm{x}_r)
    \end{bmatrix} \\
    & \approx - \bm{Q}^T e^{- \bm{\Lambda}t} \bm{Q}  \cdot
    \begin{bmatrix}
     g(\bm{x}_b) \\
    f(\bm{x}_r)
    \end{bmatrix},
\end{align*}
which is equivalent to
\begin{align}
    \label{eq: training_error}
    \bm{Q} \left(  \begin{bmatrix}
    \frac{d u(\bm{x}_b, {\bm \theta}(t))}{dt}\\
    \frac{d u_{xx}(\bm{x}_r, {\bm \theta}(t))}{dt}
    \end{bmatrix} -
     \begin{bmatrix}
     g(\bm{x}_b) \\
    f(\bm{x}_r)
    \end{bmatrix}
      \right) \approx  -  e^{- \bm{\Lambda}t} \bm{Q}  \cdot
    \begin{bmatrix}
     g(\bm{x}_b) \\
    f(\bm{x}_r)
    \end{bmatrix}.
\end{align}
This implies that the $i$-th component of the left hand side in equation \ref{eq: training_error} will decay approximately at the rate $e^{- \lambda_i t}$. In other words, the eigenvalues of the kernel characterize how fast the absolute training error decreases. Particularly, components of the target function that correspond to kernel eigenvectors with larger eigenvalues will be learned faster.  For fully-connected networks, the eigenvectors corresponding to higher eigenvalues of the NTK matrix generally exhibit lower frequencies \cite{rahaman2019spectral,ronen2019convergence,basri2020frequency}. From Figure \ref{fig: eigval_diff_a}, one can observe that the eigenvalues of the NTK of PINNs decay rapidly. This results in extremely slow convergence to the high-frequency components of the target
function. Thus we may conclude that PINNs suffer from the spectral bias either.


More generally, the NTK of PINNs after $t$ steps of gradient descent is given by
\begin{align*}
   \bm{K}(t) =
    \begin{bmatrix}
     \bm{K}_{uu}(t) & \bm{K}_{ur}(t) \\
     \bm{K}_{ru}(t) & \bm{K}_{rr}(t)
    \end{bmatrix}
    =
   \begin{bmatrix}
                        \bm{J}_u(t) \\
                        \bm{J}_r(t)
                        \end{bmatrix}
                         \begin{bmatrix}
                        \bm{J}_u^T(t) , \bm{J}_r^T(t)
                        \end{bmatrix}
                        = \bm{J}(t)\bm{J}^T(t).
\end{align*}
It follows that
\begin{align*}
     \sum_{i=1}^{N_b + N_r} \lambda_i(t) &= Tr\left(\bm{K}(t)\right)
     = Tr\left( \bm{J}(t)\bm{J}^T(t)\right) =  Tr\left( \bm{J}^T(t )\bm{J}(t)\right) \\
     &= Tr\left( \bm{J}_u^T(t) \bm{J}_u(t) + \bm{J}_r^T(t) \bm{J}_r(t)\right) = Tr \left( \bm{J}_u(t) \bm{J}^T_u(t)\right) + Tr \left( \bm{J}_r(t) \bm{J}_r^T(t) \right) \\
     &=  \sum_{i=1}^{N_b} \lambda^{uu}_i(t) + \sum_{i=1}^{N_r} \lambda^{rr}_i(t),
\end{align*}
where $\lambda_i(t), \lambda^{uu}_i(t)$ and $\lambda^{rr}_i(t)$ denote the eigenvalues of $\bm{K}(t),\bm{K}_{uu}(t)$ and $\bm{K}_{rr}(t)$, respectively. This reveals that the overall convergence rate of the total training error is characterized by the eigenvalues of $\bm{K}_{uu}$
and $\bm{K}_{rr}$ together. Meanwhile, the separate training error of $u(\bm{x}_b,\bm{\theta})$ and $u_{xx}(\bm{x}_r,\bm{\theta})$ is determined by the eigenvalues of  $\bm{K}_{uu}$
and  $\bm{K}_{rr}$, respectively. The above observation motivates us to give the following definition.
\begin{definition}
\label{def: average_convergence_rate}
For a positive semi-definite kernel matrix $\bm{K} \in \R^{n \times n}$, the average convergence rate $c$ is defined as the mean of all its eigenvalues $\lambda_i$'s, i.e.
\begin{align}
    c = \frac{\sum_{i=1}^n \lambda_i}{n} = \frac{Tr(K)}{n}.
\end{align}
In particular, for any two kernel matrices $\bm{K}_1, \bm{K}_2$ with average convergence rate $c_1$ and $c_2$ respectively, we say that $\bm{K}_1$ dominates $\bm{K}_2$ if $c_1 \gg c_2$.
\end{definition}

As a concrete example, we train a fully-connected neural network with one hidden layer and $100$ neurons to solve the model problem \ref{sec: convergence_NTK} with a fabricated solution $u(x) = \sin(a \pi  x)$ for different frequency amplitudes $a$. Figure \ref{fig: eigval_diff_a} shows all eigenvalues of $\bm{K}, \bm{K}_{uu}$ and $\bm{K}_{rr}$ at initialization in descending order.
As with conventional deep fully-connected networks, the eigenvalues of the PINNs' NTK decay rapidly and most of the eigenvalues are near zero. Moreover, the distribution of eigenvalues of $\bm{K}$ looks similar for different frequency functions (different $a$), which may heuristically explain that PINNs tend to learn all frequencies almost simultaneously, as observed in Lu {\em et. al.} \cite{lu2019deepxde}.

Another key observation here is that the eigenvalues of $\bm{K}_{rr}$ are much greater than $\bm{K}_{uu}$, namely $\bm{K}_{rr}$ dominates $\bm{K}_{uu}$ by definition \ref{def: average_convergence_rate}. As a consequence, the PDE residual converges much faster than fitting the PDE boundary conditions, which may prevent the network from approximating the correct solution. From the authors' experience, high frequency functions typically lead to high eigenvalues in $\bm{K}_{rr}$, but in some cases $\bm{K}_{uu}$ can dominate $\bm{K}_{rr}$. We believe that such a discrepancy between $\bm{K}_{uu}$ and $\bm{K}_{rr}$ is one of the key fundamental reasons behind why PINNs can often fail to train and yield accurate predictions. In light of this evidence, in the next section, we describe a practical technique to address this pathology by appropriately assigning weights to the different terms in a PINNs loss function.

\begin{figure}
    \centering
    \includegraphics[width=0.8\textwidth]{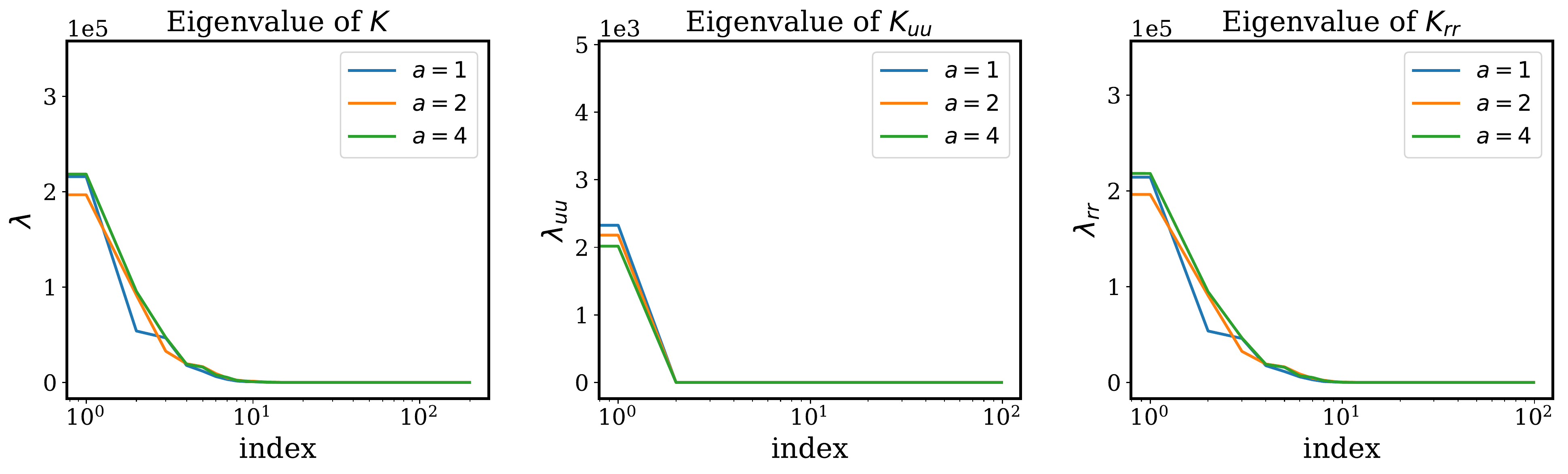}
    \caption{{\em Model problem (1D Poisson equation):} The eigenvalues of $\bm{K}, \bm{K}_{uu}$ and $\bm{K}_{rr}$ at initialization in descending order for different fabricated solutions $u(x) = \sin(a \pi x)$ where $a =1,2,4$. }
    \label{fig: eigval_diff_a}
\end{figure}

\section{Practical insights}
In this section, we consider gerenal PDEs of the form \ref{eq: PDE} - \ref{eq: PDE_BC} by leveraging the NTK theory we developed for PINNs. We approximate the latent solution $u(\bm{x})$ by a fully-connected neural network $u(\bm{x}, \bm{\theta})$  with multiple hidden layers, and train its parameters $\bm{\theta}$ by minimizing the following composite loss function
\begin{align}
    \label{eq: PINN_loss_prac}
    \mathcal{L}(\bm{\theta}) &= \mathcal{L}_b(\bm{\theta}) + \mathcal{L}_r(\bm{\theta}) \\
    & = \frac{\lambda_b}{2 N_b} \sum_{i=1}^{N_b}  |u(\bm{x}_b^i, \bm{\theta}) - g(\bm{x}_b^i)   |^2  + \frac{\lambda_r}{2 N_r} \sum_{i=1}^{N_r} |r(\bm{x}_r^i, \bm{\theta})|^2,
\end{align}
where $\lambda_b$ and $\lambda_r$ are some hyper-parameters which may be tuned manually or automatically by utilizing the back-propagated gradient statistics during training \cite{wang2020understanding}. Here, the training data $\{ {\bm{x}_b^i, g(\bm{x}_b^i)} \}_{i=1}^{N_b}$ and  $\{ {\bm{x}_r^i, f(\bm{x}_b^i)}$ may correspond to the full data-batch or mini-batches that are randomly sampled at each iteration of gradient descent.

Similar to the proof of Lemma \ref{lemma: PINN_ode}, we can derive the dynamics of the outputs $u(\bm{x},\bm{\theta})$ and $u(\bm{x}, \bm{\theta})$ corresponding to the above loss function as
\begin{align}
    \label{eq: PINN_ode_practice}
   \begin{bmatrix}
    \frac{d u(\bm{x}_b, {\bm \theta}(t))}{dt}\\
    \frac{d \mathcal{L}u(\bm{x}_r, {\bm \theta}(t))}{dt}
    \end{bmatrix}
    &=
       - \begin{bmatrix}
   \frac{\lambda_b}{N_b}  \bm{K}_{uu}(t) & \frac{\lambda_r}{N_r} \bm{K}_{ur}(t) \\
    \frac{\lambda_b}{N_b}  \bm{K}_{ru}(t) &  \frac{\lambda_r}{N_r} \bm{K}_{rr}(t)
    \end{bmatrix}
    \cdot
       \begin{bmatrix}
    u(\bm{x}_b, {\bm \theta}(t)) - g(\bm{x}_b) \\
    \mathcal{L}u(\bm{x}_r, {\bm \theta}(t)) - f(\bm{x}_r)
    \end{bmatrix}   \\
    &: = \widetilde{\bm{K}}(t)  \cdot
       \begin{bmatrix}
    u(\bm{x}_b, {\bm \theta}(t)) - g(\bm{x}_b) \\
    \mathcal{L}u(\bm{x}_r, {\bm \theta}(t)) - f(\bm{x}_r)
    \end{bmatrix},
\end{align}
where $\bm{K}_{uu}, \bm{K}_{ur}$ and $\bm{K}_{rr}$ are defined to be the same as in equation \ref{eq: NTK_PINN}. From simple stability analysis of a gradient descent (i.e. forward Euler \cite{moin2010fundamentals}) discretization of above ODE system, the maximum learning rate should be less than or equal to $2 /\lambda_{\max}(\widetilde{\bm{K}}(t))$. Also note that an alternative mechanism for controlling stability is to increase the batch size, which effectively corresponds to decreasing the learning rate.
Recall that the current setup in the main theorems put forth in this work holds for the model problem in equation \ref{eq: model problem} and fully-connected networks of one hidden layer with an NTK parameterization. This implies that, for general nonlinear PDEs, the NTK of PINNs may not remain fixed during training. Nevertheless, as mentioned in Remark \ref{remark: NTK_general}, we emphasize that, given an infinitesimal learning rate, equation \ref{eq: PINN_ode_practice} holds for any network architecture and any differential operator. Similarly, the singular values of NTK $\widetilde{\bm{K}}(t)$ determine the convergence rate of the training error using singular value decomposition, since $\widetilde{\bm{K}}(t)$ may not necessarily be semi-positive definite. Therefore, we can still understand the training dynamics of PINNs by tracking their NTK $\widetilde{\bm{K}}(t)$ during training, even for general nonlinear PDE problems.

A key observation here is that the magnitude of $\lambda_b, \lambda_r$, as well as the size of mini-batch would have a crucial impact on the singular values of $\widetilde{\bm{K}}(t)$, and, thus, the convergence rate of the training error of $u(\bm{x}_b, \bm{\theta})$ and $\mathcal{L}u(\bm{x}_r, \bm{\theta})$.
For instance, if we increase $\lambda_b$ and fix the batch size $N_b, N_r$ and the weight $\lambda_r$, then this will improve the convergence rate of $u(\bm{x}_b, \bm{\theta})$.
Furthermore, in the sense of convergence rate, changing the weights $\lambda_b$ or $\lambda_r$ is equivalent to changing the corresponding batch size $N_b, N_r$. Based on these observations, we can overcome the discrepancy between $\bm{K}_{uu}$ and $\bm{K}_{rr}$ discussed in section \ref{sec: spectral_bias} by calibrating the weights or batch size such that each component of of $u(\bm{x}_r, \bm{\theta})$ and $\mathcal{L}u(\bm{x}_r, \bm{\theta})$ has
similar convergence rate in magnitude. Since manipulating the batch size may involve extra computational costs (e.g., it may result to prohibitively very large batches), here we fix the batch size and just consider adjusting the weights $\lambda_b$ or $\lambda_r$ according to Algorithm \ref{alg: adatpive_weights}.

\begin{algorithm}
\SetAlgoLined
Consider a physics-informed neural network $u(\bm{x}, \bm{\theta})$ with parameters $\bm{\theta}$
and a loss function
\begin{align*}
\mathcal{L}(\bm{\theta}) :=   \lambda_b \mathcal{L}_b(\bm{\theta}) + \lambda_r \mathcal{L}_r(\bm{\theta}),
\end{align*}
where $\mathcal{L}_r(\bm{\theta})$ denotes the PDE residual loss and $\mathcal{L}_r(\theta)$ corresponds to  boundary conditions.
$\lambda_b = \lambda_r=1$  are free parameters used to overcome the discrepancy between $\bm{K}_{uu}$ and $\bm{K}_{rr}$.  Then use $S$ steps of a gradient descent algorithm to update the parameters $\bm{\theta}$ as:

 \For{$n = 1, \dots, S$}{
  (a) Compute $\lambda_b$ and $\lambda_r$ by
  \begin{align}
        \label{eq: lambda_b_update}
      & \lambda_b =  \frac{\sum_{i=1}^{N_r + N_b}\lambda_i(n)}{\sum_{i=1}^{N_b}\lambda_i^{uu}(n)} = \frac{Tr(\bm{K}(n))}{Tr(\bm{K}_{uu}(n))} \\
        \label{eq: lambda_r_update}
      &\lambda_r =  \frac{\sum_{i=1}^{N_r + N_b}\lambda_i(n)}{\sum_{i=1}^{N_r}\lambda_i^{rr}(n)} = \frac{Tr(\bm{K}(n))}{Tr(\bm{K}_{rr}(n))}
  \end{align}
  where $\lambda_i(n), \lambda_i^{uu}$ and $\lambda_i^{rr}(n)$ are eigenvalues of $\bm{K}(n), \bm{K}_{uu}(n), \bm{K}_{rr}(n)$ at $n$-th iteration.

  (b) Update the parameters $\theta$ via gradient descent
  \begin{align}
  \label{eq:theta_update}
      \bm{\theta}_{n+1} = \bm{\theta}_{n} - \eta \nabla_{\bm{\theta}}\mathcal{L}(\bm{\theta}_n)
  \end{align}
}
\caption{Adaptive weights for physics-informed neural networks}
\label{alg: adatpive_weights}
\end{algorithm}

First we remark that  the updates in equations \ref{eq: lambda_b_update} and \ref{eq: lambda_r_update} can either take place at every iteration of the gradient descent loop, or at a frequency specified by the user (e.g., every 10 gradient descent steps). To compute the sum of eigenvalues, it suffices to compute the trace of the corresponding NTK matrices, which can save some computational resources. Besides, we point out that the computation of the NTK $\bm{K}(t)$ is associated with the training data points fed to the network at each iteration, which means that the values of the kernel are not necessarily same at each iteration. However, if we assume that all training data  points are sampled from the same distribution and the change of NTK at each iteration is negligible, then the computed kernel should be approximately equal up to a permutation matrix. As a result,  the change of eigenvalues of $\bm{K}(t)$ at each iteration is also negligible and thus  the training process of Algorithm \ref{alg: adatpive_weights} should be stable. In section \ref{sec: alg_adaptive_weights}, we performed detailed numerical experiments to validate  the effectiveness of the proposed algorithm.

Here we also note that, in previous work, Wang {\em et. al.} introduced an alternative empirical approach for automatically tuning the weights $\lambda_b$ or $\lambda_r$ with the goal of balancing the magnitudes of the back-propagated gradients originating from different terms in a PINNs loss function. While effective in practice, this approach lacked any theoretical justification and did not provide a deeper insight into the training dynamics of PINNs. In contrast, the approach presented here follows naturally
from the NTK theory derived in section \ref{sec: PINN_dynamics}, and aims to trace and tackle the
the pathological convergence behavior of PINNs at its root.

\section{Numerical Experiments}
In this section, we provide a series of numerical studies that aim to validate our theory or access the performance of the proposed algorithm against the standard PINNs \cite{raissi2018deep} for inferring the solution of PDEs. Throughout numerical experiments we will approximate the latent variables by fully-connected neural networks with NTK parameterization \ref{eq: NTK_param} and  hyperbolic tangent activation functions. All networks are trained using standard stochastic gradient descent, unless otherwise specified. Finally, all results presented in this section can be reproduced using
our publicly available code  \url{https://github.com/PredictiveIntelligenceLab/PINNsNTK}.

\subsection{Convergence of the NTK of PINNs}
\label{sec: convergence_NTK}
As our first numerical example, we still focus on the model problem \ref{eq: model problem}  and verify the convergence of the PINNs' NTK. Specifically, we set $\Omega$ to be the unit interval $[0,1]$ and fabricate the exact solution to this problem taking the form $ u(x) = \sin(\pi x)$. The corresponding $f$ and $g$ are given by
\begin{align*}
   & f(x) = - \pi^2 \sin(\pi x), \quad x \in [0, 1] \\
   & g(x) = 0, \quad x = 0,1.
\end{align*}
We proceed by approximating the latent solution $u(x)$ by a fully-connected neural network $u(x, \bm{\theta})$ of one hidden layer with NTK parameterization (see equation \ref{eq: NTK_param}), and a hyperbolic tangent activation function.
The corresponding loss function is given by
\begin{align}
     \mathcal{L}(\bm{\theta}) &= \mathcal{L}_b(\bm{\theta}) + \mathcal{L}_r(\bm{\theta}) \\
     &= \frac{1}{2} \sum_{i=1}^{N_b}  |u(\bm{x}_b^i, \bm{\theta}) - g(\bm{x}_b^i)   |^2
     + \frac{1}{2} \sum_{i=1}^{N_r}  |u_{xx}(\bm{x}_r^i, \bm{\theta}) - f(\bm{x}_r^i)   |^2
\end{align}
Here we choose $N_b = N_r = 100$ and the collocation points $\{x_r^i\}_{i=1}^{N_r}$ are uniformly spaced in the unit interval. To monitor the change of the NTK $\bm{K}(t)$ for this PINN model, we train the network for different widths and for $10,000$ iterations by minimizing the loss function given above using standard full-batch gradient descent with a learning rate of $10^{-5}$. Here we remark that, in order to keep the gradient descent dynamics  \ref{eq: PINN_ode} steady,
the learning rate should be less than $2 / \lambda_{\max}$, where $\lambda_{\max}$ denotes the maximum eigenvalue of $\bm{K}(t)$.

Figure  \ref{fig: weight_change_one_hidden_layer} and  \ref{fig: kernel_change_one_hidden_layer} present the relative change in the norm of network's weights and NTK (starting from a random initialization) during training. As it can be seen, the change of both the weights and the NTK tends to zero as the width of the network grows to infinity, which is consistent with  Lemma \ref{lemma: weight_change_little} and  Theorem \ref{theorem: kernel_constant}. Moreover, we know that convergence in a matrix norm implies convergence in eigenvalues, and eigenvalues characterize the properties of a given matrix. To this end, we compute and monitor all eigenvalues of $\bm{K}$(t) of the network for width$=500$ at initialization and  after $10,000$ steps of gradient and plot them in descending order in Figure \ref{fig: eigval_change_width_500}. As expected, we see that all eigenvalues barely change for these two snapshots. Based on these observations, we may conclude that the NTK of PINNs with one hidden layer stays almost fixed during training.

However, PINNs of multiple hidden layers are not covered by our theory at the moment. Out of interest, we also investigate the relative change of weights, kernel, as well as the kernel's eigenvalues for a fully-connected network with three hidden layers (see Figure \ref{fig: three_hidden_layer}). We can observe that the change in both the weights and the NTK behaves almost identical to the case of a fully-connected network with one hidden layer shown in Figure \ref{fig: one_hidden_layer}.
Therefore we may conjecture that, for any linear or even nonlinear PDEs, the NTK of PINNs converges to a deterministic kernel and remains constant during training in the infinite width limit.

\begin{figure}
     \centering
     \begin{subfigure}[b]{0.3\textwidth}
         \centering
         \includegraphics[width=\textwidth]{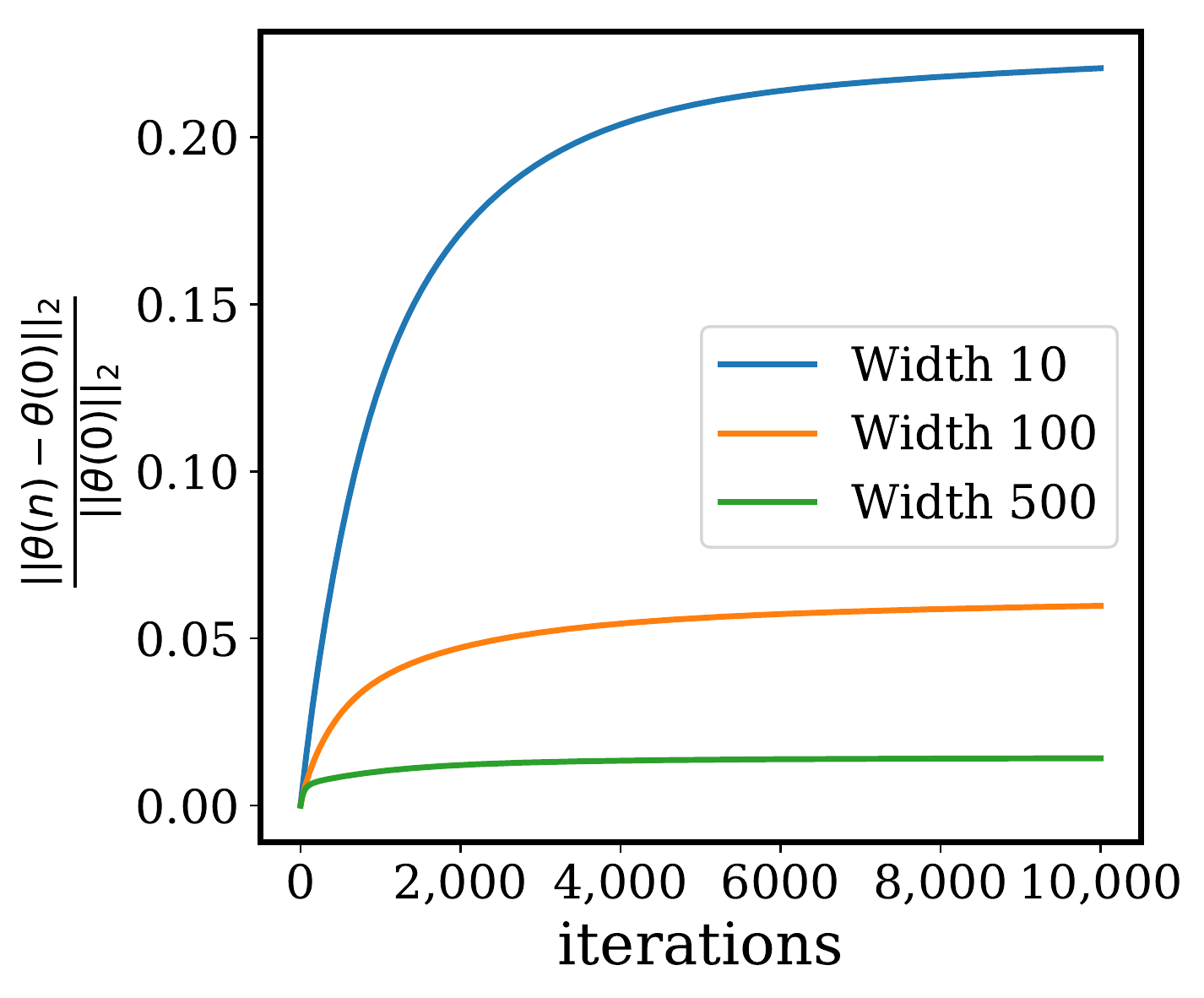}
         \caption{}
         \label{fig: weight_change_one_hidden_layer}
     \end{subfigure}
     \begin{subfigure}[b]{0.3\textwidth}
         \centering
         \includegraphics[width=\textwidth]{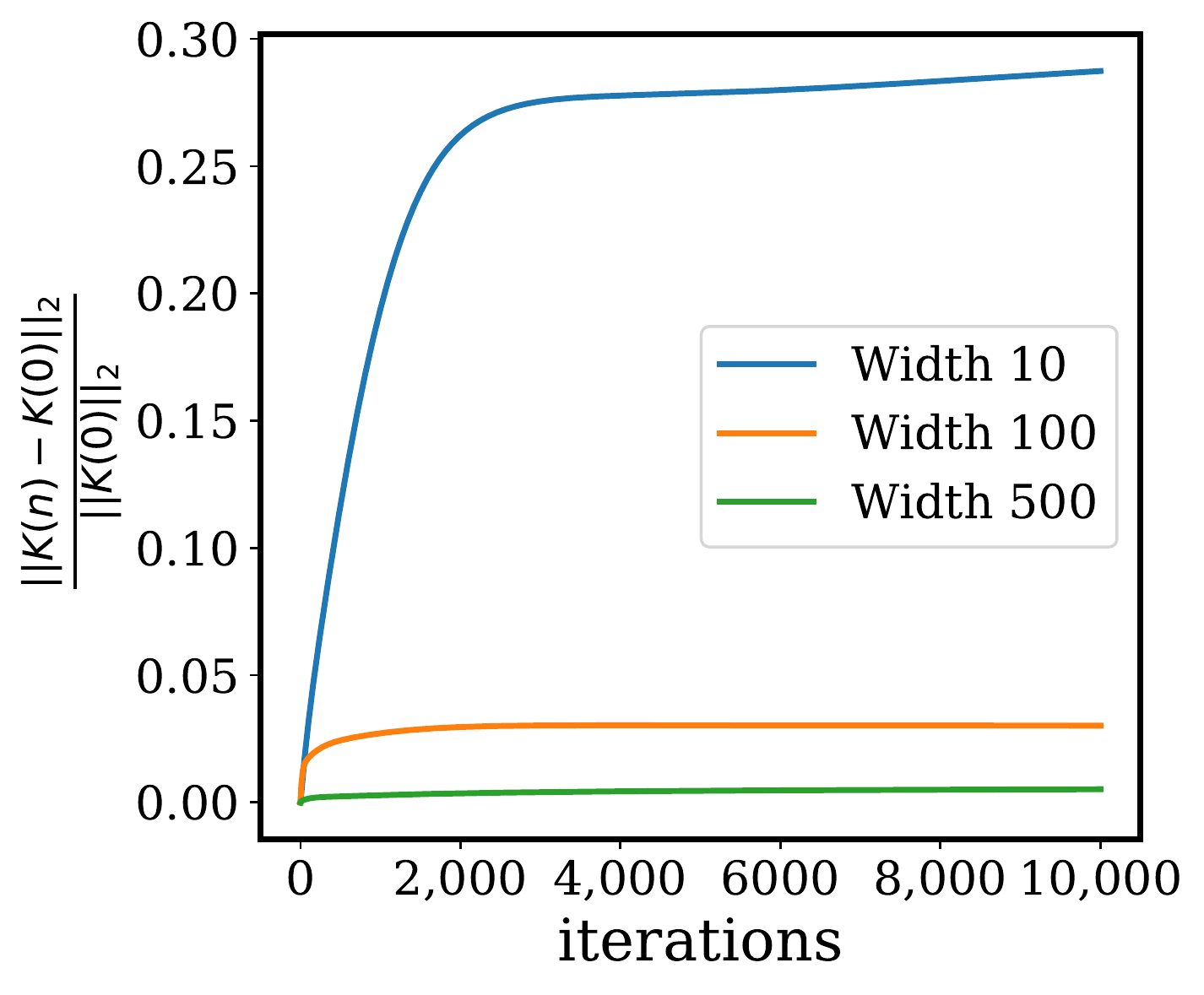}
         \caption{}
         \label{fig: kernel_change_one_hidden_layer}
     \end{subfigure}
      \begin{subfigure}[b]{0.3\textwidth}
         \centering
         \includegraphics[width=\textwidth]{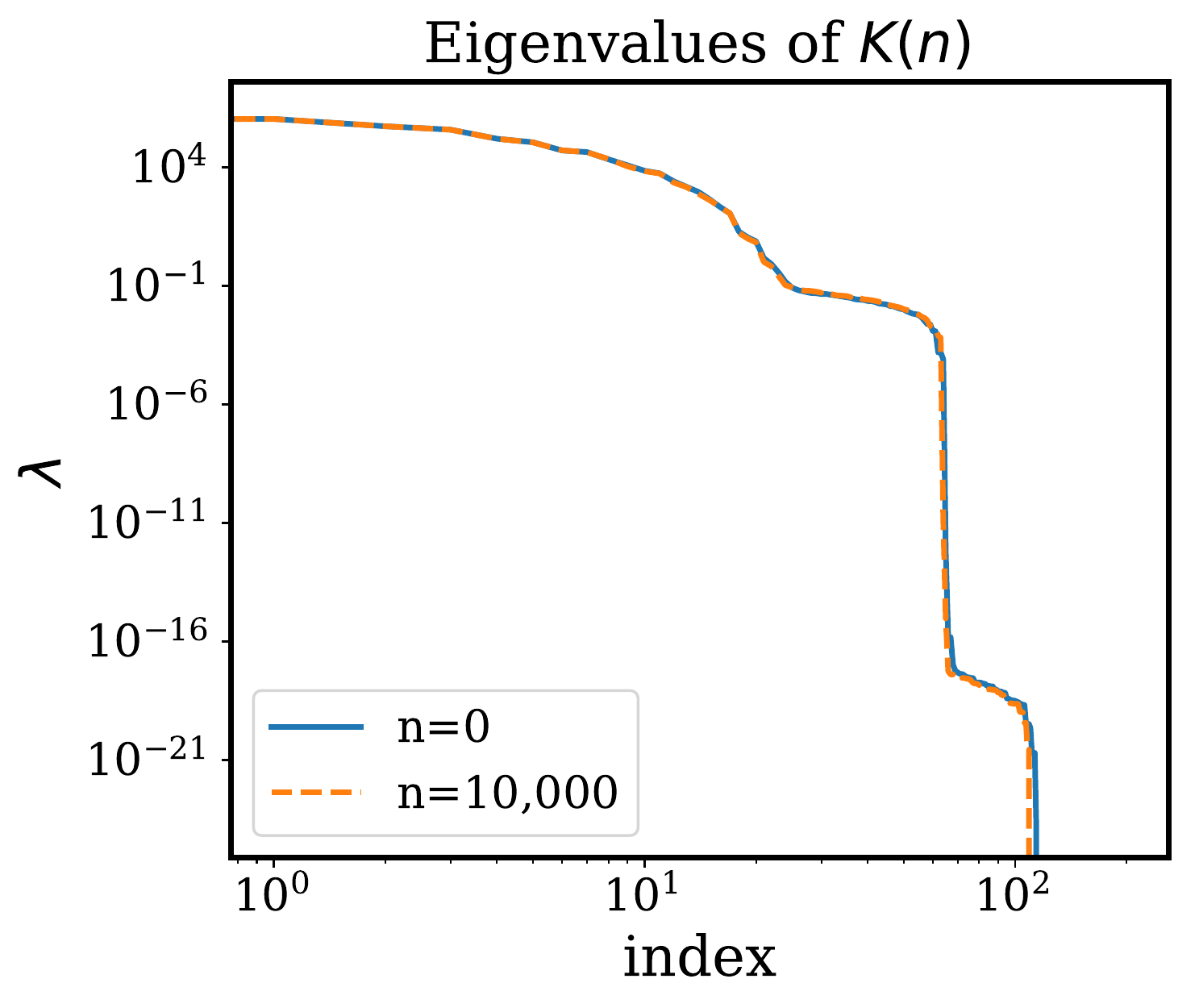}
         \caption{}
         \label{fig: eigval_change_width_500}
     \end{subfigure}
        \caption{{\em Model problem \ref{sec: convergence_NTK} (1D Poisson equation):} (a) (b) The relative change of  parameters $\bm{\theta}$ and the NTK of PINNs $\bm{K}$  obtained by training a fully-connected neural network with one hidden layer and different widths ($10, 100, 500$) via $10,000$ iterations of full-batch gradient descent with a learning rate of $10^{-5}$. (c) The eigenvalues of the NTK $\bm{K}$ at initialization and at the last step of training ($n = 10,000$).}
        \label{fig: one_hidden_layer}
\end{figure}

\begin{figure}
     \centering
     \begin{subfigure}[b]{0.3\textwidth}
         \centering
         \includegraphics[width=\textwidth]{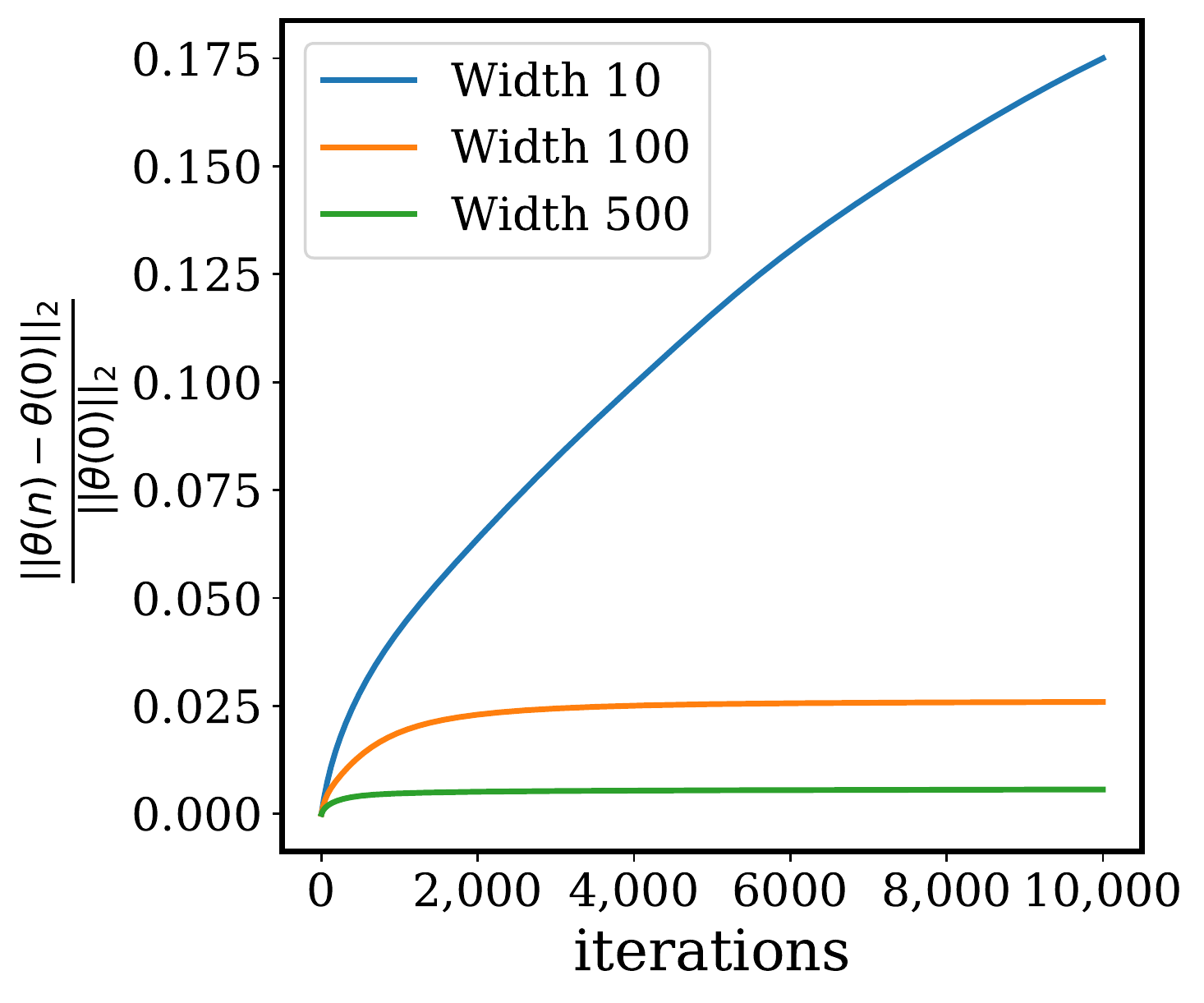}
         \caption{}
         \label{fig: weight_change_3_hidden_layer}
     \end{subfigure}
     \begin{subfigure}[b]{0.3\textwidth}
         \centering
         \includegraphics[width=\textwidth]{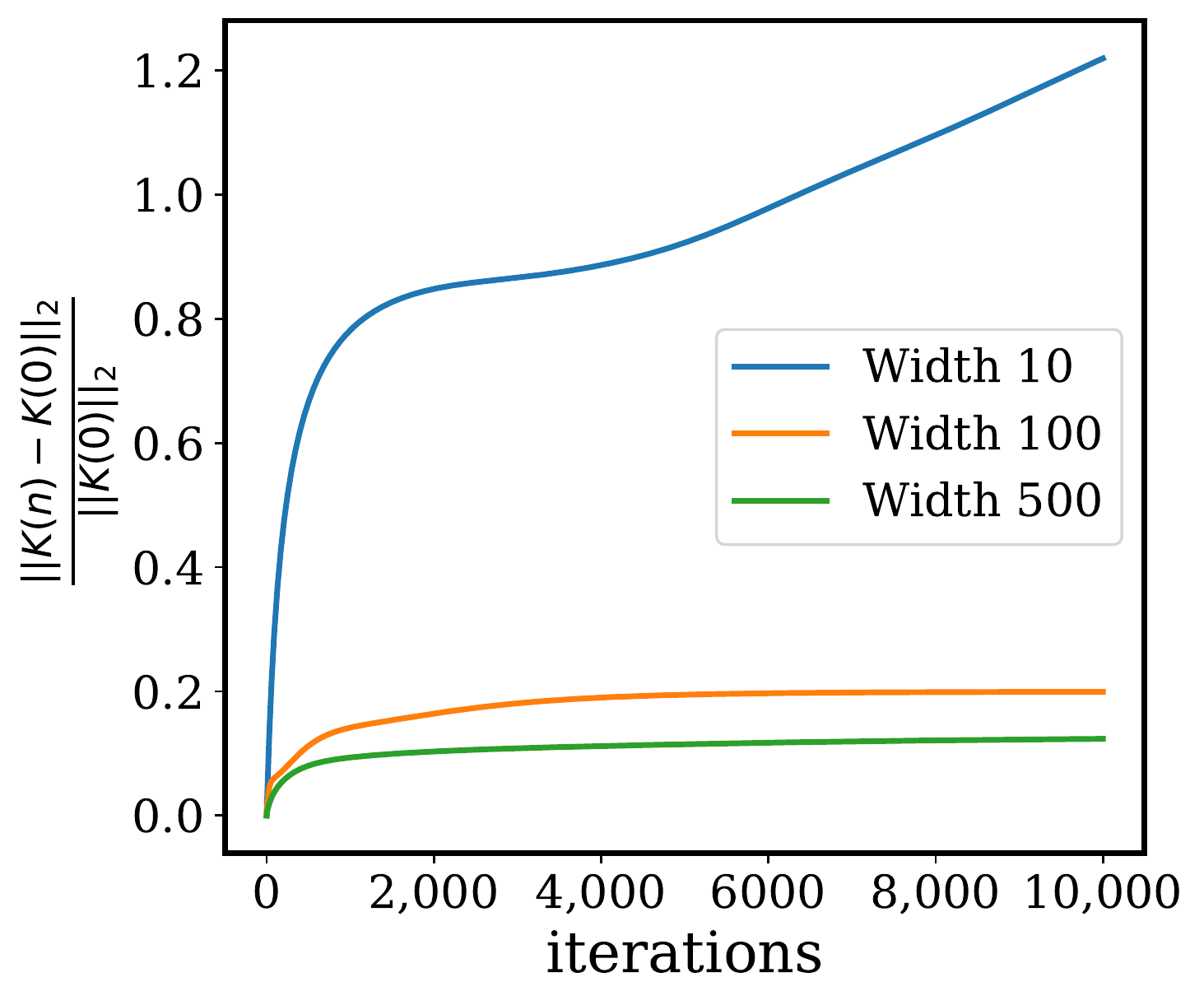}
         \caption{}
         \label{fig: kernel_change_3_hidden_layer}
     \end{subfigure}
      \begin{subfigure}[b]{0.3\textwidth}
         \centering
         \includegraphics[width=\textwidth]{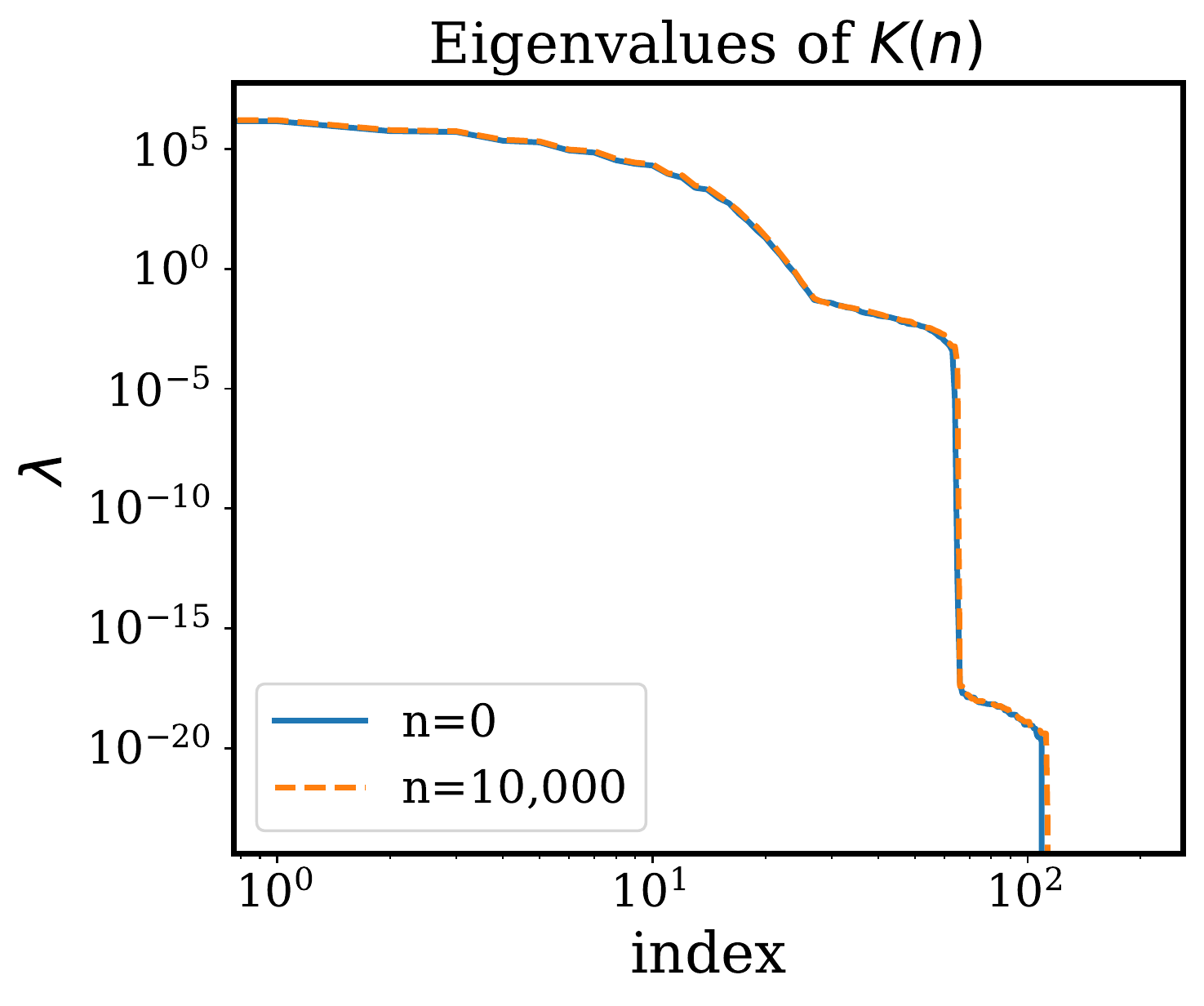}
         \caption{}
         \label{fig: eigval_change_3_hidden_layer}
     \end{subfigure}
         \caption{{\em Model problem \ref{sec: convergence_NTK} (1D Poisson equation):} (a) (b) The relative change of  parameters $\bm{\theta}$ and the NTK of PINNs $\bm{K}$  obtained by training a fully-connected neural network with three hidden layers and different widths ($10, 100, 500$) via $10,000$ iterations of full-batch gradient descent with a learning rate of $10^{-5}$. (c) The eigenvalues of the NTK $\bm{K}$ at initialization and at the last step of training ($n = 10,000$).}
        \label{fig: three_hidden_layer}
\end{figure}

\subsection{Adaptive training for PINNs}
\label{sec: alg_adaptive_weights}


In this section, we aim to validate the developed theory and examine the effectiveness of the proposed adaptive training algorithm on the model problem of equation \ref{sec: convergence_NTK}. To this end, we consider a fabricated exact solution of the form $u(x) = \sin (4 \pi x)$, inducing a corresponding forcing term $f$ and Dirichlet boundary condition $g$ given by
\begin{align*}
    & f(x) = - 16 \pi^2 \sin (4 \pi x), \quad x \in [0,1]\\
    & g(x) = 0, \quad x = 0,1.
\end{align*}
We  proceed by approximating the latent solution $u(x)$ by a fully-connected neural network with one hidden layer and width set to $100$.  Recall from  Theorem  \ref{theorem: NTK_PINNs_init} and Theorem \ref{theorem: kernel_constant}, that the NTK barely changes during training. This implies that the weights $\lambda_b, \lambda_r$ are determined by NTK at initialization and thus they can be regarded as fixed weights during training. Moreover, from Figure \ref{fig: eigval_diff_a}, we already know that $\bm{K}_{rr}$ dominates $\bm{K}_{uu}$ for this example. Therefore, the updating rule for hyper-parameters $\lambda_b,\lambda_r$ at $t$ step of gradient descent can be reduced to
\begin{align}
    \label{eq: fix_weights_1}
        &\lambda_b = \frac{\sum_{i=1}^{N_b+N_r} \lambda_i(t)}{\sum_{i=1}^{N_b} \lambda_i^{uu}(t)} \approx \frac{\sum_{i=1}^{N_r} \lambda_i^{rr}(t)}{\sum_{i=1}^{N_b} \lambda_i^{uu}(t)} \approx \frac{Tr(\bm{K}_{rr}(0))}{Tr(\bm{K}_{uu}(0))} \\
        \label{eq: fix_weights_2}
    & \lambda_r =  \frac{\sum_{i=1}^{N_b+N_r} \lambda_i(t)}{\sum_{i=1}^{N_r} \lambda_i^{rr}(t)} \approx 1.
\end{align}
We proceed by training the network via full-batch gradient descent with a learning rate of $10^{-5}$ to minimize the following loss function
\begin{align*}
      \mathcal{L}(\bm{\theta}) &= \mathcal{L}_b(\bm{\theta}) + \mathcal{L}_r(\bm{\theta}) \\
     &= \frac{\lambda_b}{2N_b} \sum_{i=1}^{N_b}  |u(\bm{x}_b^i, \bm{\theta}) - g(\bm{x}_b^i)   |^2
     + \frac{\lambda_r}{2N_r} \sum_{i=1}^{N_r}  |u_{xx}(\bm{x}_r^i, \bm{\theta}) - f(\bm{x}_r^i N_r)|^2,
\end{align*}
where the batch sizes are $N_b = N_r = 100$, $\lambda_r =1$ and the computed $\lambda_b \approx 100$.

A comparison of predicted solution $u(x)$ between the original PINNs ($\lambda_b =\lambda_r = 1$)  and PINNs with adaptive weights after $40,000$ iterations are shown in figure \ref{fig: comparsion_pred_u}. It can be observed that the proposed algorithm yields a much more accurate predicted solution and improves the relative $L^2$ error by about two orders of magnitude. Furthermore, we also investigate how the predicted performance of PINNs depends on the choice of different weights in the loss function. To this end, we fix $\lambda_r =1$ and train the same network, but now we manually tune $\lambda_b$. Figure \ref{fig: l2_error_diff_weights} presents a visual assessment of relative $L^2$ errors of predicted solutions for different $\lambda_b \in [1, 500]$ averaged over ten independent trials. One can see that the relative $L^2$ error decreases rapidly
to a local minimum as $\lambda_b$ increases from $1$ to about $100$ and then shows oscillations as $\lambda_b$ continues to increase. Moreover, a large magnitude of $\lambda_b$ seems to lead to a large standard deviation in the $L^2$ error, which may be due to the imaginary eigenvalues of the indefinite kernel $\widetilde{\bm{K}}$ resulting in an unstable training process.
This empirical simulation study confirms that the weights $\lambda_r =1$ and $\lambda_b$ suggested by our theoretical analysis based on analyzing the NTK spectrum are robust and closely agree with the optimal weights obtained via manual hyper-parameter tuning.

\begin{figure}
     \centering
     \begin{subfigure}[b]{0.8\textwidth}
         \centering
         \includegraphics[width=\textwidth]{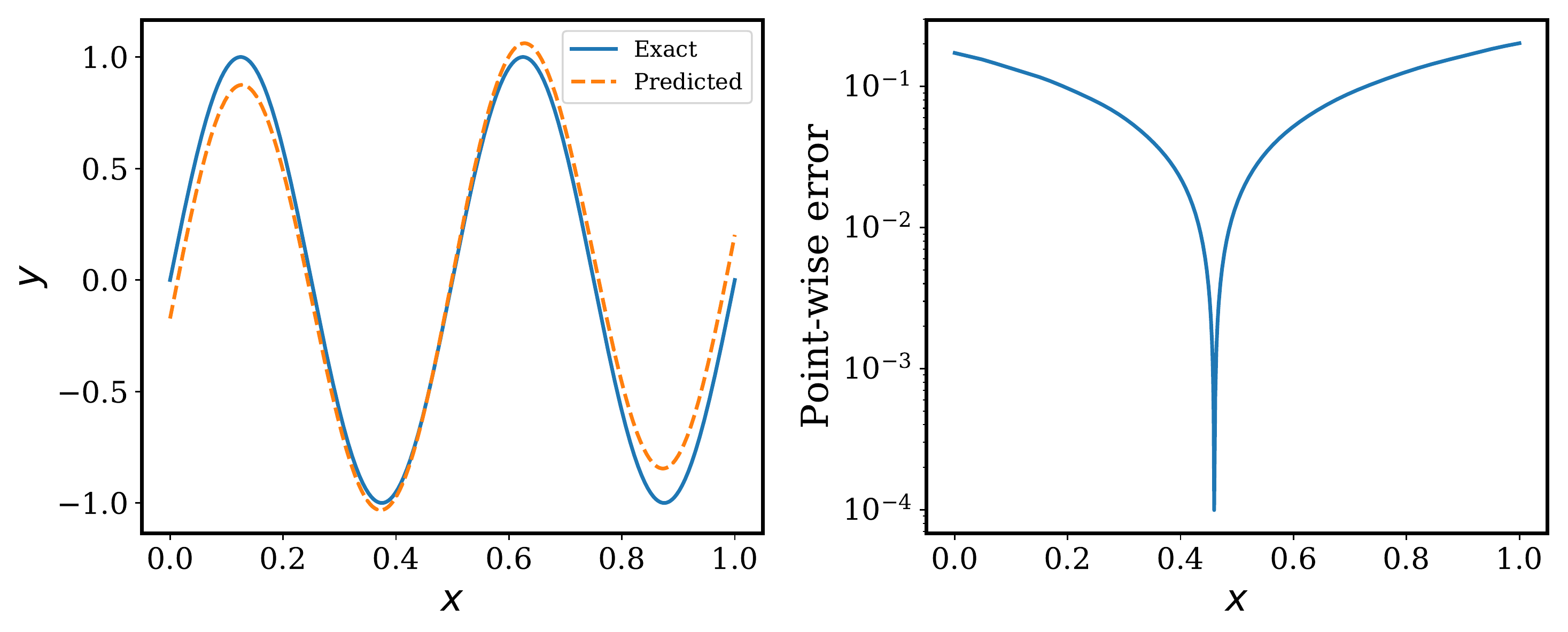}
         \caption{}
         \label{fig: pred_u_original}
     \end{subfigure}
     \begin{subfigure}[b]{0.8\textwidth}
         \centering
         \includegraphics[width=\textwidth]{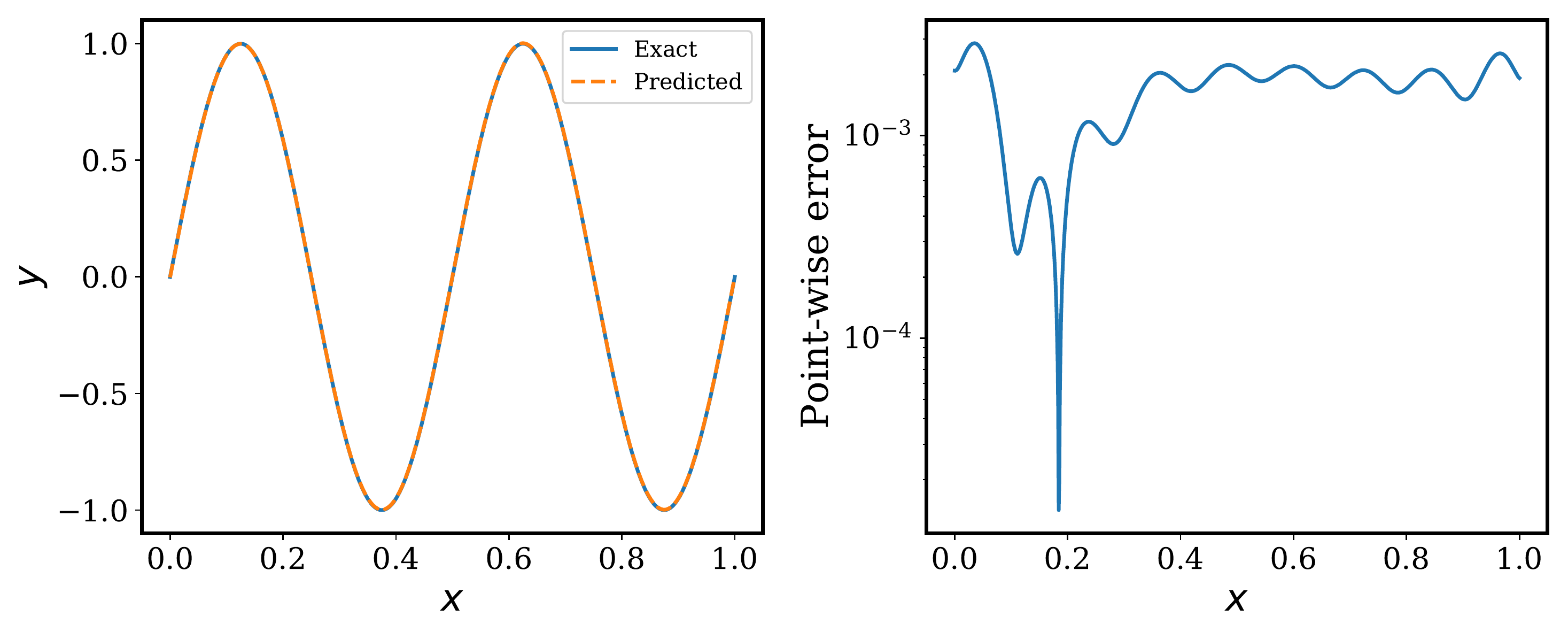}
         \caption{}
         \label{fig: pred_u_adaptive}
     \end{subfigure}
         \caption{{\em Model problem \ref{sec: alg_adaptive_weights} (1D Poisson equation):} (a) The predicted solution against the exact solution obtained by training a fully-connected neural network of one hidden layer with width $=100$ via $40,000$ iterations of full-batch gradient descent with a learning rate of $10^{-5}$ . The relative $L^2$ error is $2.40e-01$. (b) The predicted solution against the exact solution obtained by training the same neural network using fixed weights $\lambda_b= 100, \lambda_r =1$ via $40,000$ iterations of full-batch gradient descent with a learning rate of $10^{-5}$. The relative $L^2$ error is $1.63e-03$.}
        \label{fig: comparsion_pred_u}
\end{figure}

\begin{figure}
    \centering
    \includegraphics[width=0.4\textwidth]{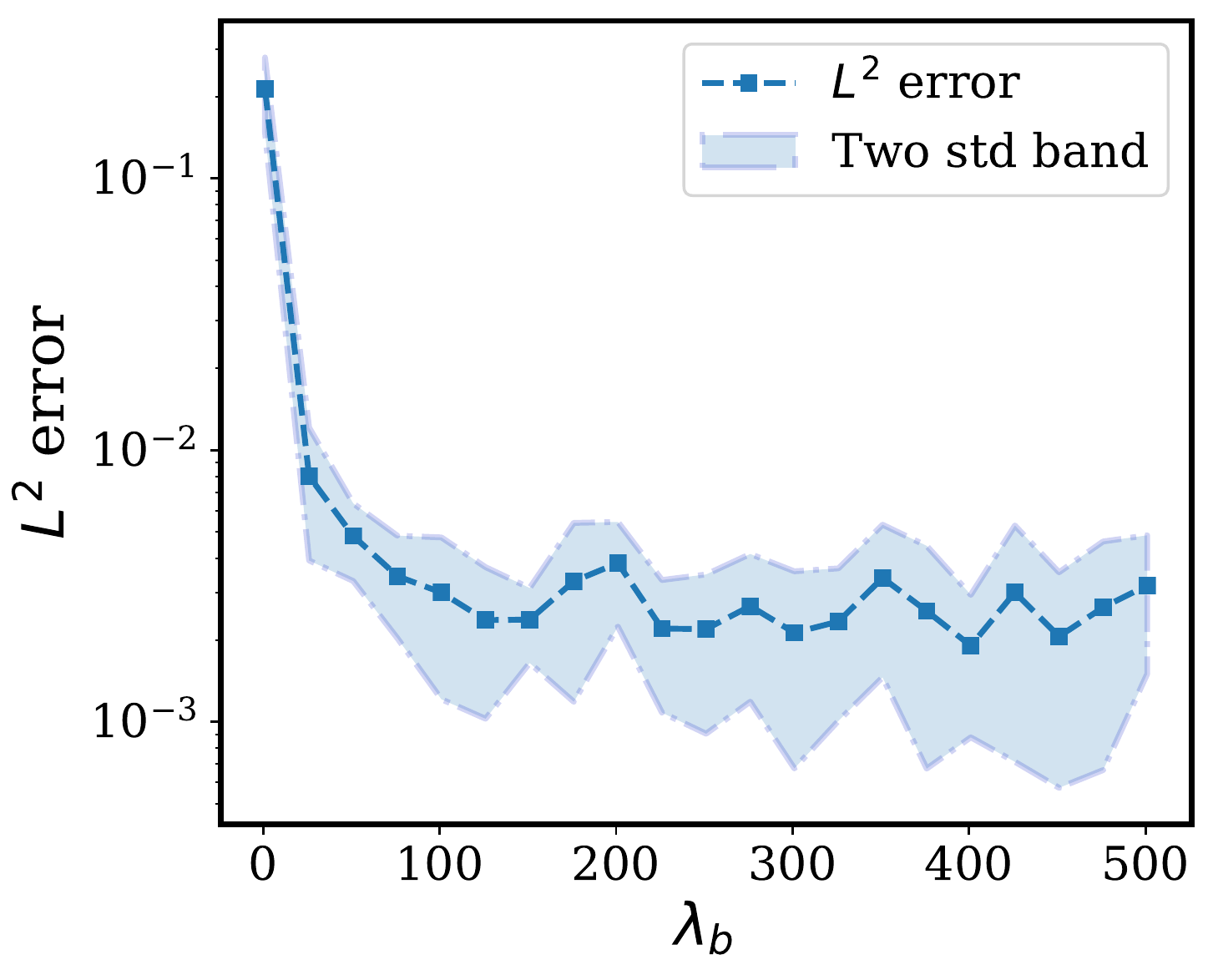}
    \caption{{\em Model problem of equation \ref{sec: alg_adaptive_weights} (1D Poisson equation):} The relative $L^2$ error of predicted solutions averaged over 10 independent trials by training a fully-connected neural network of one hidden layer with width $=100$ using different fixed weights   $\lambda_b \in [1, 500]$ for $40,000$ gradient descent iterations.}
    \label{fig: l2_error_diff_weights}
\end{figure}

\subsection{One-dimensional wave equation}
As our last example, we present a study that demonstrates the effectivenss of Algorithm  \ref{alg: adatpive_weights} in a practical problem for which conventional PINNs models face severe diffuculties.
To this end, we consider a one-dimensional wave equation in the domain $\Omega = [0,1 ] \times [0,1]$ taking the the form
\begin{align}
    \label{eq: wave_pde}
    &u_{tt}(x, t) - 4 u_{xx}(x, t) = 0, \quad (x, t) \in (0,1) \times (0,1) \\
      \label{eq: wave_bc}
    & u(0, t) = u(1, t) = 0, \quad t \in [0,1] \\
    \label{eq:   wave_u_ic}
    & u(x,0) = \sin( \pi x) + \frac{1}{2} \sin(4 \pi x), \quad x\in [0,1 ]\\
    \label{eq:   wave_ut_ic}
    & u_t(x,0) = 0, \quad x\in [0,1 ].
\end{align}
First, by d'Alembert's formula \cite{evans1998partial}, the solution $u(x,t)$ is given by
\begin{align}
    u(x, t) =  \sin( \pi x) \cos(2 \pi t) + \frac{1}{2} \sin(4 \pi x) \cos(8\pi t).
\end{align}
Here we treat the temporal coordinate $t$ as an additional spatial coordinate in $\bm{x}$ and then the initial condition \ref{eq:   wave_u_ic} can be included in the boundary condition \ref{eq: wave_bc}, namely
\begin{align*}
    u(\bm{x}) = g(\bm{x}), \quad x \in \partial \Omega
\end{align*}
Now we approximate the solution $u$ by a 5-layer deep fully-connected network $u(\bm{x}, \bm{\theta})$ with $500$ neurons per hidden layer, where $\bm{x} = (x,t)$. Then we can formulate a ``physics-informed" loss function by
\begin{align}
    \label{eq: loss_wave}
    \mathcal{L}(\bm{\theta}) &=  \mathcal{L}_u(\bm{\theta}) +  \mathcal{L}_{u_t}(\bm{\theta}) +  \mathcal{L}_r(\bm{\theta}) \\
                             &=  \frac{\lambda_u}{2 N_u} \sum_{i=1}^{N_u} |u(\bm{x}_u^i,\bm{\theta}) - g(\bm{x}_u^i)  |^2
                             +  \frac{\lambda_{u_t}}{2 N_{u_t}} \sum_{i=1}^{N_{u_t}} |u_t(\bm{x}_{u_t}^i,\bm{\theta})   |^2
                             +  \frac{\lambda_r}{2 N_r} \sum_{i=1}^{N_r} | \mathcal{L} u(\bm{x}_r^i,\bm{\theta}) |^2,
 \end{align}
where the hyper-parameters  $\lambda_u, \lambda_{u_t},\lambda_r$ are initialized to $1$, the batch sizes are set to $N_u = N_{u_t} = N_r = 300$, and $\mathcal{L} = \partial_{tt} - 4 \partial_{xx}$. Here
all training data are uniformely sampling inside the computational domain at each gradeint descent iteration. The network $u(\bm{x}, \bm{\theta})$ is initialized using the standard Glorot scheme \cite{glorot2010understanding} and then trained by minimizing the above loss function via stochastic gradient descent using the Adam optimizer with default settings \cite{kingma2014adam}. Figure \ref{fig: wave_pred_u_original} provides a comparison between the predicted solution against the ground truth obtained after $80,000$ training iterations. Clearly the original PINN model fails to approximate the ground truth solution and the relative $L^2$ error is above $40 \%$.

To explore the reason behind PINN's failure for this example, we compute its NTK and track it during training. Similar to the proof of Lemma \ref{lemma: PINN_ode}, the corresponding NTK can be derived from the loss function \ref{eq: loss_wave}
\begin{align}
    \label{eq: PINN_NTK_wave}
   \begin{bmatrix}
    \frac{d u(\bm{x}_u, {\bm \theta}(t))}{dt}\\
     \frac{d u_t(\bm{x}_{u_t}, {\bm \theta}(t))}{dt} \\
    \frac{d \mathcal{L}u(\bm{x}_r, {\bm \theta}(t))}{dt}
    \end{bmatrix}
    &: = \widetilde{\bm{K}}(t)  \cdot
       \begin{bmatrix}
    u(\bm{x}_b, {\bm \theta}(t)) - g(\bm{x}_b) \\
      u_t(\bm{x}_{u_t}, {\bm \theta}(t)) \\
    \mathcal{L}u(\bm{x}_r, {\bm \theta}(t))
    \end{bmatrix},
\end{align}
where
\begin{align*}
    &\widetilde{\bm{K}}(t) = \begin{bmatrix}
      \frac{\lambda_u}{N_u} \bm{J}_u(t) \\
       \frac{\lambda_{u_t}}{N_{u_t}} \bm{J}_{u_t}(t) \\
        \frac{\lambda_r}{N_r}  \bm{J}_r(t)
      \end{bmatrix}
        \cdot \begin{bmatrix}
          \bm{J}^T_u(t) , \bm{J}_{u_t}^T(t) , \bm{J}^T_r(t)
            \end{bmatrix},  \\
       & \left[\bm{K}_u(t)\right]_{ij} = \left[\bm{J}_u(t) \bm{J}^T_u(t)\right]_{ij} =
       \left\langle   \frac{d u({\bm x}_u^i, {\bm \theta}(t))}{d{\bm \theta}},  \frac{d u(\bm{x}_u^j, {\bm \theta}(t))}{d{\bm \theta}}            \right\rangle  \\
      & \left[\bm{K}_{u_t}(t)\right]_{ij} = \left[\bm{J}_{u_t}(t) \bm{J}^T_{u_t}(t)\right]_{ij} = 
       \left\langle   \frac{d u_t({\bm x}_{u_t}^i, {\bm \theta}(t))}{d{\bm \theta}},  \frac{d u(\bm{x}_{u_t}^j, {\bm \theta}(t))}{d{\bm \theta}}            \right\rangle  \\
      & \left[\bm{K}_r(t)\right]_{ij} = \left[\bm{J}_r(t) \bm{J}^T_r(t)\right]_{ij} = 
       \left\langle   \frac{d\mathcal{L} u({\bm x}_r^i, {\bm \theta}(t))}{d{\bm \theta}},  \frac{d\mathcal{L} u(\bm{x}_r^j, {\bm \theta}(t))}{d{\bm \theta}}            \right\rangle.
\end{align*}
A visual assessment of the eigenvalues of $\bm{K}_u, \bm{K}_{u_t}$ and $\bm{K}_r$ at initialization and the last step of gradient descent are presented in Figure \ref{fig: wave_NTK_eigval}. It can be observed that the NTK does not remain fixed and all eigenvalues move ``outward" in the beginning of the training, and then remain almost static such that $\bm{K}_r$ and $\bm{K}_{u_t}$ dominate $\bm{K}_u$ during training. Consequently, the components of $u_t(\bm{x}_{u_t},\bm{\theta})$ and $\mathcal{L}u(\bm{x}_r,\bm{\theta}))$ converge much faster than the  loss of boundary conditions, and, therefore, introduce a severe discrepancy in the convergence rate of each different term in the loss, causing this standard PINNs model to collapse. To verify our hypothesis, we also train the same network using Algorithm \ref{alg: adatpive_weights} with
the following generalized updating rule for hyper-parameters $\lambda_u, \lambda_{u_t}$ and $\lambda_r$
\begin{align}
    &\lambda_u = \frac{Tr(\bm{K}_u) + Tr(\bm{K}_{u_t}) + Tr(\bm{K}_r) }{Tr(\bm{K}_u)} \\
    &\lambda_{u_t} = \frac{Tr(\bm{K}_u) + Tr(\bm{K}_{u_t}) + Tr(\bm{K}_r) }{Tr(\bm{K}_{u_t})} \\
        &\lambda_r = \frac{Tr(\bm{K}_u) + Tr(\bm{K}_{u_t}) + Tr(\bm{K}_r) }{Tr(\bm{K}_r)}.
\end{align}
In particular, we update these weights every $1,000$ training iterations, hence the extra computational costs compared to a standard PINNs approach is negligible. The results of this experiment are shown in Figure \ref{fig: wave_pred_u_adaptive}, from which one can easily see that the predicted solution obtained using the proposed adaptive training scheme achieves excellent agreement with the ground truth and the relative $L^2$ error is $1.73e-3$. To quantify the effect of the hyper-parameters $\lambda_u,\lambda_{u_t}$ and $\lambda_r$ on the NTK, we also compare the eigenvalues of $\bm{K}_u, \bm{K}_{u_t}$ and $\bm{K}_r$ multiplied with or without the hyper-parameters at last step of gradient descent. As it can be seen in Figure \ref{fig: wave_lambda_K_change}, the discrepancy of the convergence rate of different components in total training errors is considerably resolved. Furthermore, Figure \ref{fig: wave_lambda_change} presents the change of weights during training  and we can see that $\lambda_u,\lambda_{u_t}$ increase rapidly and then remain almost fixed while $\lambda_r$ is near 1 for all time. So we may conclude that the overall training process using Algorithm \ref{alg: adatpive_weights} is stable.

\begin{figure}
     \centering
     \begin{subfigure}[b]{0.8\textwidth}
         \centering
         \includegraphics[width=\textwidth]{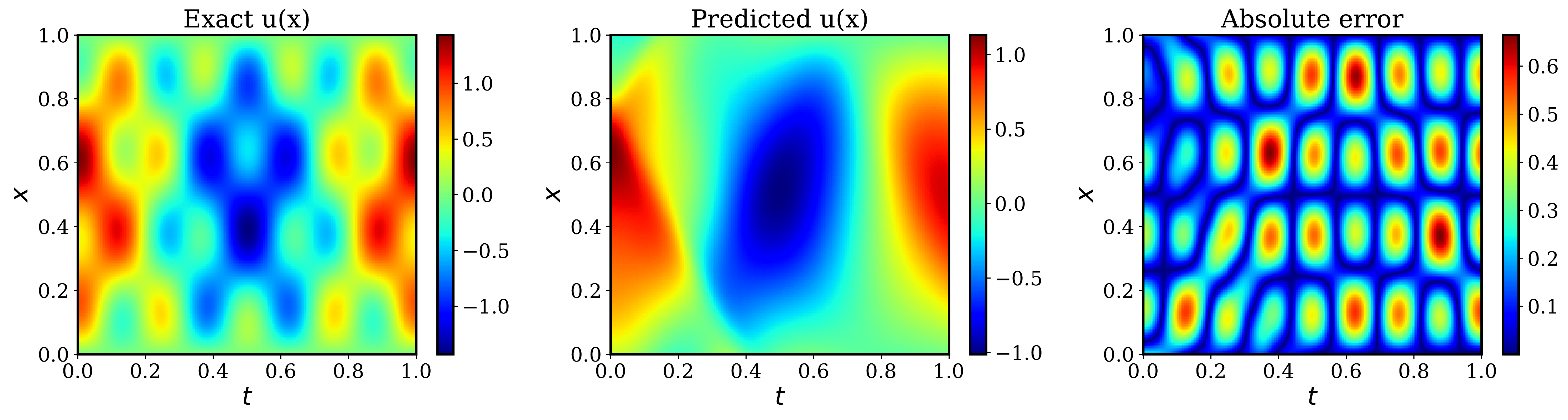}
         \caption{}
         \label{fig: wave_pred_u_original}
     \end{subfigure}
     \begin{subfigure}[b]{0.8\textwidth}
         \centering
         \includegraphics[width=\textwidth]{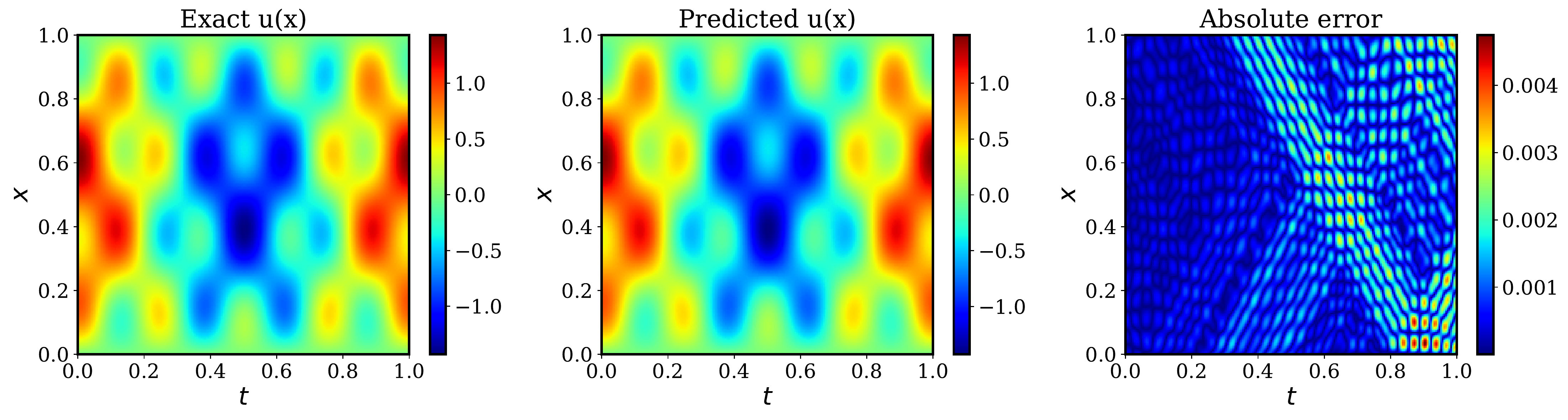}
         \caption{}
         \label{fig: wave_pred_u_adaptive}
     \end{subfigure}
        \caption{{\em One-dimensional wave equation:} (a) The predicted solution versus the exact solution by training a fully-connected neural network with five hidden layers and $500$ neurons per layer using the Adam optimizer with default settings \cite{kingma2014adam} after $80,000$ iterations. The relative $L^2$ error is $4.518e-01$. (b)  The predicted solution versus the exact solution by training the same network using Algorithm \ref{alg: adatpive_weights} after $80,000$ iterations. The relative $L^2$ error is $1.728e-03$.}
        \label{fig: wave_pred_u}
\end{figure}

\begin{figure}
    \centering
    \includegraphics[width=0.8\textwidth]{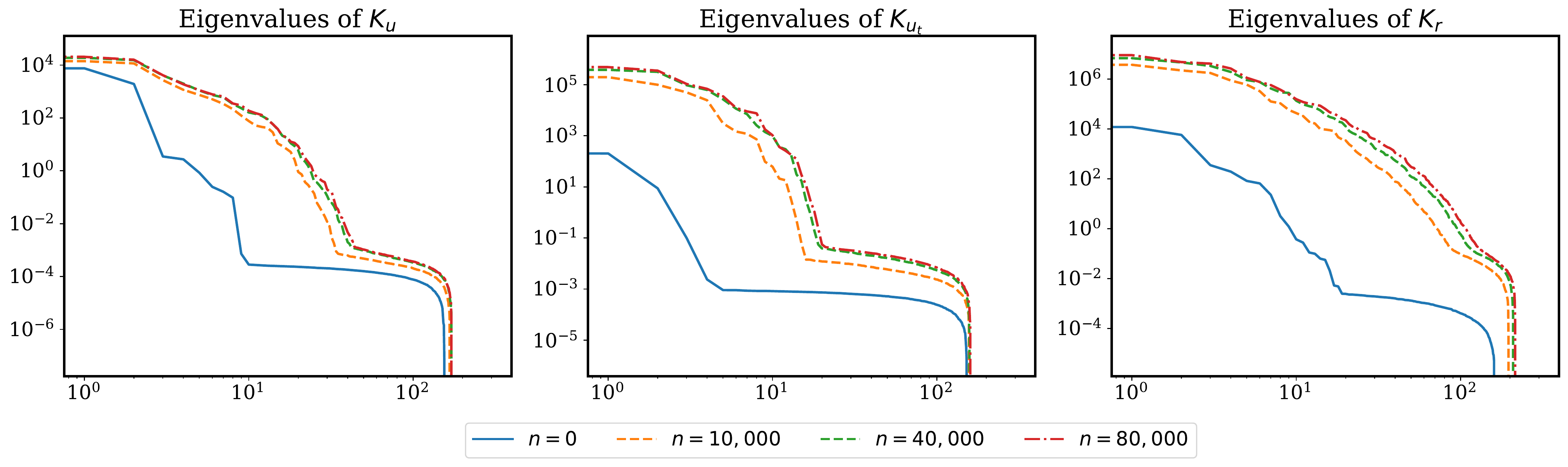}
    \caption{{\em One-dimensional wave equation:} Eigenvalues of $\bm{K}_{u}, \bm{K}_{u_t}$ and $\bm{K}_r$ at different snapshots during training, sorted in descending order.}
    \label{fig: wave_NTK_eigval}
\end{figure}

\begin{figure}
     \centering
     \begin{subfigure}[b]{0.3\textwidth}
         \centering
         \includegraphics[width=\textwidth]{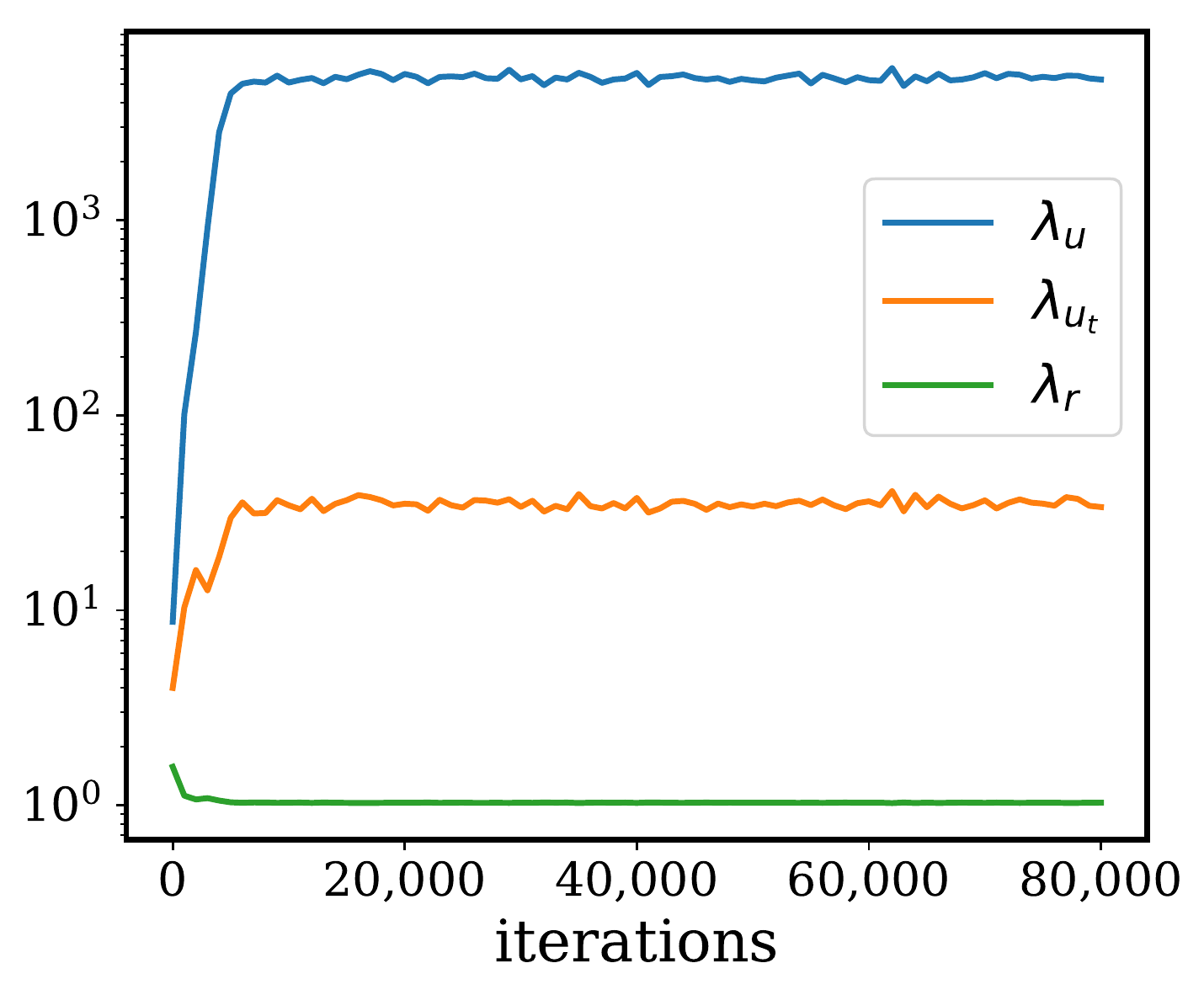}
         \caption{ }
         \label{fig: wave_lambda_change}
     \end{subfigure}
     \begin{subfigure}[b]{0.3\textwidth}
         \centering
         \includegraphics[width=\textwidth]{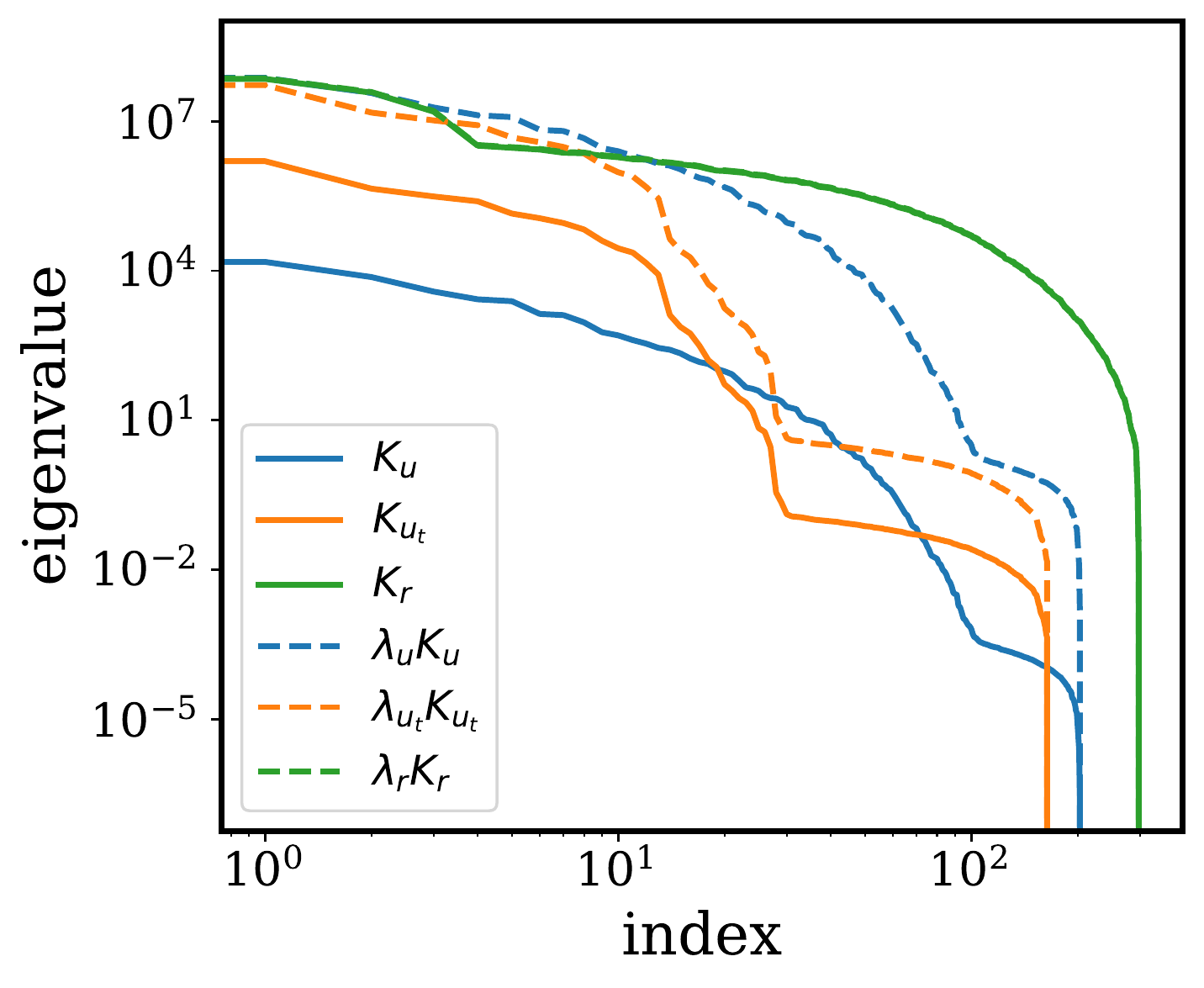}
         \caption{}
         \label{fig: wave_lambda_K_change}
     \end{subfigure}
        \caption{{\em One-dimensional wave equation:} (a) The evolution of hyper-parameters $\lambda_u, \lambda_{u_t}$ and $\lambda_r$ during training of a five-layer deep fully-connected neural network with $500$ neurons per layer using Algorithm \ref{alg: adatpive_weights}. (b) The eigenvalues of $\bm{K}_{u}, \bm{K}_{u_t}, \bm{K}_r$ and $\lambda_u \bm{K}_{u}, \lambda_{u_t}\bm{K}_{u_t}, \lambda_r\bm{K}_r$ at last step of training. }
\end{figure}

\section{Discussion}

This work has produced a novel theoretical understanding of physics-informed neural networks
by deriving and analyzing their limiting neural tangent kernel. Specifically,
we first show that infinitely wide physics-informed neural networks under the NTK parameterization converge to Gaussian processes. Furthermore, we derive NTK of PINNs and show that, under suitable assumptions, it converges to a deterministic kernel and barely changes during training as the width of the network grows to infinity. To provide further insight, we analyze the training dynamics of fully-connected PINNs through the lens of their NTK and show that not only they suffer from spectral bias, but they also exhibit a discrepancy in the convergence rate among the different loss components contributing to the total training error. To resolve this discrepancy, we propose a novel algorithm such that the coefficients of different terms in a PINNs' loss function can be dynamically updated according to balance the average convergence rate of different components in the total training error. Finally, we carry out a series of numerical experiments to verify our theory and validate the effectiveness of the proposed algorithms.

Although this work takes an essential step towards understanding PINNs and their training dynamics, there are many open questions worth exploring. Can the proposed NTK theory for PINNs be extended fully-connected networks with multiple hidden layers, nonlinear equations, as well as the neural network architectures such as convolutional neural networks, residual networks, etc.?
To which extend do these architecture suffer from spectral bias or exhibit similar discrepancies in their convergence rate? In a parallel thrust, it is well-known that PINNs perform much better for inverse problems than for forward problems, such as the ones considered in this work. Can we incorporate the current theory to analyze inverse problems and explain they are better suited to PINNs? Moreover, going beyond vanilla gradient descent dynamics, how do the training dynamics of PINNs and their corresponding NTK evolve via gradient descent with momentum (e.g. Adam \cite{kingma2014adam})? In practice, despite some improvements in the performance of PINNs brought by assigning appropriate weights to the loss function, we emphasize that such methods cannot change the distribution of eigenvalues of the NTK and, thus, cannot directly resolve spectral bias. Apart from this, assigning weights may result in indefinite kernels which can have imaginary eigenvalues and thus yield unstable training processes. Therefore, can we come up with better methodologies to resolve spectral bias using specialized network architectures, loss functions, etc.?
We believe that answering these questions not only paves a new way to better understand PINNs and its training dynamics, but also opens a new door for developing scientific machine learning algorithms with provable convergence guarantees, as needed for many critical applications in computational science and engineering.

\section*{Acknowledgements}
This work received support from the US Department of Energy under the Advanced Scientific Computing Research program (grant DE-SC0019116) and the Air Force Office of Scientific Research (grant FA9550-20-1-0060).







\bibliographystyle{unsrt}

\bibliography{references}

\appendix
\section{Proof of Lemma \ref{lemma: PINN_ode}}
\label{sec: proof_lemma_PINN_ode}

\begin{proof}
Recall that for given training data $\{\bm{x}_b^i, g(\bm{x}_b^i)\}_{i=1}^{N_b}, \{\bm{x}_r^i, f(\bm{x}_r^i)\}_{i=1}^{N_r}$, the loss function is given by
\begin{align*}
     \mathcal{L}(\bm{\theta}) &= \mathcal{L}_b(\bm{\theta}) + \mathcal{L}_r(\bm{\theta}) \\
     & = \frac{1}{2} \sum_{i=1}^{N_b}  |u(\bm{x}_b^i, \bm{\theta}) - g(\bm{x}_b^i)   |^2  +  \frac{1}{2} \sum_{i=1}^{N_r} |r(\bm{x}_r^i, \bm{\theta})|^2.
\end{align*}
Now let us consider the corresponding gradient flow
\begin{align*}
        \frac{d {\bm \theta}}{d t} = - \nabla_{\bm \theta} \mathcal{L}(\bm{\theta}) = - [\sum_{i = 1}^{N_b} (u(\bm{x}_b^i,{\bm \theta}(t)) -g(\bm{x}_b^i)) \frac{\partial u}{\partial \bm{\theta}}(\bm{x}_b^i,{\bm \theta}(t))  +\sum_{i = 1}^{N_r} (\mathcal{L}(\bm{x}_r^i,{\bm \theta}(t)) -f(\bm{x}_r^i))\frac{\partial \mathcal{L}u}{\partial \bm{\theta}}(\bm{x}_r^i,{\bm \theta}(t))].
\end{align*}
It follows that for $0 \leq j \leq N_b$,
\begin{align*}
    & \frac{d u(\bm{x}_b^j, {\bm \theta}(t))}{dt} =   \frac{d u(\bm{x}_b^j, {\bm \theta}(t))}{d{\bm \theta}}^T  \cdot \frac{d {\bm \theta}}{d t} \\
    & = - \frac{d u(\bm{x}_b^j, {\bm \theta}(t))}{d{\bm \theta}}^T \cdot \Big[\sum_{i = 1}^{N_b} (u(\bm{x}_b^i,{\bm \theta}(t)) -g(\bm{x}_b^i)) \frac{\partial u}{\partial \bm{\theta}}(\bm{x}_b^i,{\bm \theta}(t))  +\sum_{i = 1}^{N_r} (\mathcal{L}(\bm{x}_r^i,{\bm \theta}(t)) -f(\bm{x}_r^i))\frac{\partial \mathcal{L}u}{\partial \bm{\theta}}(\bm{x}_r^i,{\bm \theta}(t))\Big] \\
    &= - \sum_{i = 1}^{N_b} (u(\bm{x}_b^i,{\bm \theta}) -g(\bm{x}_b^i)) \Big\langle  \frac{d u(\bm{x}_b^i, {\bm \theta}(t))}{d{\bm \theta}},  \frac{d u(\bm{x}_b^j, {\bm \theta}(t))}{d{\bm \theta}}    \Big\rangle
    \\
    &-  \sum_{i = 1}^{N_r} (\mathcal{L}u(\bm{x}_r^i,{\bm \theta}) -f(\bm{x}_r^i)) \Big\langle  \frac{\mathcal{L}u(\bm{x}_r^i, {\bm \theta}(t))}{d{\bm \theta}},  \frac{d u(\bm{x}_b^j, {\bm \theta}(t))}{d{\bm \theta}}  \Big \rangle.
\end{align*}
Similarly,
\begin{align*}
   & \frac{d \mathcal{L}u(\bm{x}_r^j, {\bm \theta}(t))}{dt} = \frac{d \mathcal{L}u(\bm{x}_r^j, {\bm \theta}(t))}{d{\bm \theta}}^T  \cdot \frac{d {\bm \theta}}{d t} \\
   &= \frac{d \mathcal{L}u(\bm{x}_r^j, {\bm \theta}(t))}{d{\bm \theta}}^T  \cdot \Big[\sum_{i = 1}^{N_b} (u(\bm{x}_b^i,{\bm \theta}(t)) -g(\bm{x}_b^i)) \frac{\partial u}{\partial \bm{\theta}}(\bm{x}_b^i,{\bm \theta}(t))  +\sum_{i = 1}^{N_r} (\mathcal{L}(\bm{x}_r^i,{\bm \theta}(t)) -f(\bm{x}_r^i))\frac{\partial \mathcal{L}u}{\partial \bm{\theta}}(\bm{x}_r^i,{\bm \theta}(t))\Big] \\
   &=  - \sum_{i = 1}^{N_b} (u(\bm{x}_b^i,{\bm \theta}) -g(\bm{x}_b^i)) \Big\langle  \frac{d u(\bm{x}_b^i, {\bm \theta}(t))}{d{\bm \theta}},  \frac{d \mathcal{L}u(\bm{x}_b^j, {\bm \theta}(t))}{d{\bm \theta}}    \Big\rangle  \\
   &-  \sum_{i = 1}^{N_r} (\mathcal{L}u(\bm{x}_r^i,{\bm \theta}) -f(\bm{x}_r^i)) \Big\langle  \frac{\mathcal{L}u(\bm{x}_r^i, {\bm \theta}(t))}{d{\bm \theta}},  \frac{d \mathcal{L}u(\bm{x}_r^j, {\bm \theta}(t))}{d{\bm \theta}}   \Big\rangle.
\end{align*}
Then we can rewritte the above equations as
\begin{align}
    &\frac{d u(\bm{x}_b, \bm{\theta}(t))}{dt} = - \bm{K}_{uu}(t) \cdot (u(\bm{x}_b, \bm{\theta}(t)) - g(\bm{x}_b)) - \bm{K}_{ur}(t) \cdot (\mathcal{L}u(\bm{x}_r, \bm{\theta}(t)) - f(\bm{x}_r))\\
    & \frac{d \mathcal{L}u(\bm{x}_r, \bm{\theta}(t))}{dt} = - \bm{K}_{ru}(t) \cdot (u(\bm{x}_b, \bm{\theta}(t)) - g(\bm{x}_b)) - \bm{K}_{rr}(t)  \cdot (\mathcal{L}u(\bm{x}_r, \bm{\theta}(t)) - f(\bm{x}_r)),
\end{align}
where $\bm{K}_{ru}(t) = \bm{K}^T_{ur}(t)$ and the $(i,j)$-th entries are given by
\begin{align*}
      & (\bm{K}_{uu})_{ij}(t) =  \Big\langle  \frac{d u({\bm x}_b^i, {\bm \theta}(t))}{d{\bm \theta}},  \frac{d u(\bm{x}_b^j, {\bm \theta}(t))}{d{\bm \theta}}    \Big\rangle \\
   & (\bm{K}_{ur})_{ij}(t)  =  \Big\langle  \frac{d u(\bm{x}_b^i, {\bm \theta}(t))}{d{\bm \theta}} , \frac{d \mathcal{L}u(\bm{x}_r^j, {\bm \theta}(t))}{d{\bm \theta}}  \Big\rangle  \\
   & (\bm{K}_{rr})_{ij}(t)  =  \Big\langle \frac{d \mathcal{L}(\bm{x}_r^i, {\bm \theta}(t))}{d{\bm \theta}},  \frac{d \mathcal{L}(\bm{x}_r^j, {\bm \theta}(t))}{d{\bm \theta}}   \Big\rangle.
\end{align*}

\end{proof}

\section{Proof of Theorem  \ref{theorem: u_xx_gaussian}}
\label{sec: proof_u_xx_gaussian}
\begin{proof}
Recall equation \ref{eq: u_xx} and that all weights and biases are initialized by independent standard Gaussian distributions. Then by the central limit theorem, we have
\begin{align*}
        u_{xx}(x, \bm{\theta}) &= \frac{1}{\sqrt{N}} \bm{W}^{(1)} \cdot \Big[  \Ddot{\sigma}(\bm{W}^{(0)}x + \bm{b}^{(0)} ) \odot \bm{W}^{(0)}  \odot \bm{W}^{(0)} \Big]  \\
        &= \frac{1}{\sqrt{N}} \sum_{k=1}^N  \bm{W}_k^{(1)} \Ddot{\sigma} (\bm{W}_k^{(0)}x + \bm{b}^{(0)}_k) (\bm{W}_k^{(0)})^2  \\
        & \xrightarrow{\mathcal{D}} \mathcal{N}(0, \Sigma(x)) \triangleq Y(x),
\end{align*}
as $N \rightarrow \infty$, where $\mathcal{D}$ denotes convergence in distribution and $Y(x)$ is a centered Gaussian random variable with covariance
\begin{align*}
    \Sigma(x) = \text{Var} [   \bm{W}_k^{(1)} \Ddot{\sigma} (\bm{W}_k^{(0)}x + \bm{b}^{(0)}_k) (\bm{W}_k^{(0)})^2 ].
\end{align*}
Since $\Ddot{\sigma}$ is bounded, we may assume that $|\Ddot{\sigma}| \leq C$. Then we have
\begin{align*}
    \sup_{N}  \mathbb{E} \Big[ |u_{xx}(x,\theta)|^2 \Big]  &= \sup_{N}  \mathbb{E} \Big[ \mathbb{E} \big[|u_{xx}(x,\theta)|^2 | \bm{W}^{(0)}, \bm{b}^{(0)}    \big]  \Big] \\
    &=  \sup_{N} \mathbb{E}\Big[ \frac{1}{N} \sum_{k=1}^N \big(\bm{W}_k^{(0)} \big)^4 \big(\Ddot{\sigma}(\bm{W}_k^{(0)}x + \bm{b}^{(0)}_k) \big)^2   \Big] \\
    &\leq C \mathbb{E} \big[ \big(\bm{W}_k^{(0)} \big)^4  \big] < \infty.
\end{align*}
This implies that $u_{xx}(x, \bm{\theta})$ is uniformly integrable with respect to $N$. Now for any given point $x,x'$, we have
\begin{align*}
     \Sigma^{(1)}_{xx}(x, x') &\triangleq  \mathbb{E} \big[Y(x)Y(x') \big]  =   \lim_{N \rightarrow \infty} \mathbb{E} \big[  u_{xx}(x,\bm{\theta}) u_{xx}(x',\bm{\theta})) \big]
     \\
     &=  \lim_{N \rightarrow \infty} \mathbb{E}\Big[ \mathbb{E}[ u_{xx}(x,\bm{\theta}) u_{xx}(x',\bm{\theta})     | \bm{W^{(0)}}, \bm{b}^{(0)} ]  \Big] \\
     & = \lim_{N \rightarrow \infty} \mathbb{E}   \Bigg[\frac{1}{N} \Big( \Ddot{\sigma}(\bm{W}^{(0)}x + \bm{b}^{(0)} ) \odot \bm{W}^{(0)}  \odot \bm{W}^{(0)} \Big)^T \Big( \Ddot{\sigma}(\bm{W}^{(0)}x' + \bm{b}^{(0)} ) \odot \bm{W}^{(0)}  \odot \bm{W}^{(0)}  \Big)  \Bigg]\\
     & =  \lim_{N \rightarrow \infty}\mathbb{E}   \Big[\frac{1}{N} \sum_{k=1}^N (\bm{W}_k^{(0)})^4 \Ddot{\sigma}(\bm{W}_k^{(0)}x + \bm{b}^{(0)}_k) \Ddot{\sigma}(\bm{W}_k^{(0)}x' + \bm{b}^{(0)}_k)   \Big]  \\
     & = \mathop{\mathbb{E}}_{u,v \sim \mathcal{N}(0,1)} \Big[ u^4 \Ddot{\sigma}(ux + v) \Ddot{\sigma}(ux' + v)   \Big].
\end{align*}
This concludes the proof.

\end{proof}

\section{Proof of Theorem \ref{theorem: NTK_PINNs_init}}
\label{sec:proof_NTK_PINNs_init}

\begin{proof}
\label{sec: proof_NTK_PINN_init}
To warm up, we first compute $\bm{K}_{uu}(0)$ and its infinite width limit, which is already covered in \cite{jacot2018neural}. By the definition of $\bm{K}_{uu}(0)$, for any two given input $x, x'$ we have
\begin{align*}
      \bm{K}_{uu}(0) =  \Big\langle  \frac{d u(x, {\bm \theta}(0))}{d{\bm \theta}},  \frac{d u(x', {\bm \theta}(0))}{d{\bm \theta}}    \Big\rangle.
\end{align*}
Recall that
\begin{align*}
     u(x, \bm{\theta}) = \frac{1}{\sqrt{N}} \bm{W}^{(1)} \cdot \sigma(\bm{W}^{(0)}x + \bm{b}^{(0)}) + \bm{b}^{(1)} = \frac{1}{\sqrt{N}} \sum_{k=1}^N \bm{W}^{(1)}_k \sigma( \bm{W}^{(0)}_kx +  \bm{b}^{(0)}_k ) + \bm{b}^{(1)},
\end{align*}
and $\bm{\theta} = (\bm{W}^{(0)} , \bm{W}^{(1)}, \bm{b}^{(0)}, \bm{b}^{(1)})$. Then we have
\begin{align*}
    &\frac{\partial u(x, \bm{\theta})}{\partial \bm{W}^{(0)}_k}=  \frac{1}{\sqrt{N}} \bm{W}^{(1)}_k
    \Dot{\sigma}(\bm{W}_k^{(0)} x + \bm{b}^{(0)}_k)x
    \\
     &\frac{\partial u(x, \bm{\theta})}{\partial \bm{W}^{(1)}_k}=  \frac{1}{\sqrt{N}} \sigma(\bm{W}_k^{(0)} x + \bm{b}^{(0)}_k) \\
     & \frac{\partial u(x, \bm{\theta})}{\partial \bm{b}^{(0)}_k} =  \frac{1}{\sqrt{N}}   \bm{W}^{(1)}_k  \Dot{\sigma}(\bm{W}_k^{(0)} x + \bm{b}^{(0)}_k) \\
     &\frac{\partial u(x, \bm{\theta})}{\partial \bm{b}^{(1)}} = 1.
\end{align*}
Then by the law of large numbers we have
\begin{align*}
     \sum_{k=1}^{N}  \frac{\partial u(x, \bm{\theta})}{\partial \bm{W}^{(0)}_k}  \frac{\partial u(x', \bm{\theta})}{\partial \bm{W}^{(0)}_k} &=  \frac{1}{N}
    \sum_{k=1}^{N}  \Big[ \bm{W}_k^{(1)} \Dot{\sigma} (\bm{W}_k^{(0)}x +  \bm{b}_k^{(0)}) x  \Big]  \cdot  \Big[          \bm{W}_k^{(1)} \Dot{\sigma} (\bm{W}_k^{(0)}x' +  \bm{b}_K^{(0)}) x'   \Big]          \\
    &=  \frac{1}{N}
    \Big( \sum_{k=1}^{N}  (\bm{W}_k^{(1)})^2 \Dot{\sigma} (\bm{W}_k^{(0)}x +  \bm{b}_k^{(0)})  \Dot{\sigma} (\bm{W}_k^{(0)}x' +  \bm{b}_k^{(0)}) \Big)  (x x') \\
    & \overset{\mathcal{P}}{\longrightarrow} \mathbb{E}\bigg[ (\bm{W}_k^{(1)})^2  \Dot{\sigma} (\bm{W}_k^{(0)}x +  \bm{b}_k^{(0)})  \Dot{\sigma} (\bm{W}_k^{(0)}x' +  \bm{b}_k^{(0)}) \bigg]  (x x') \\
    &=  \mathbb{E} \Big[ (\bm{W}_k^{(1)})^2     \Big] \mathbb{E} \Big[  \Dot{\sigma} (\bm{W}_k^{(0)}x +  \bm{b}_k^{(0)})  \Dot{\sigma} (\bm{W}_k^{(0)}x' +  \bm{b}_k^{(0)})   \Big] (x x' )
    = \Dot{\Sigma}^{(1)}(x,x')(xx'),
\end{align*}
as ${N} \rightarrow \infty$,  where $ \Dot{\Sigma}^{(1)}(x,x')$ is defined in equation \ref{eq: dot_Sigma}.

Moreover,
\begin{align*}
      \sum_{k=1}^{N} \frac{\partial u(x, \bm{\theta})}{\partial \bm{W}^{(1)}_k}  \frac{\partial u(x', \bm{\theta})}{\partial \bm{W}^{(1)}_k}  &= \frac{1}{N}\sum_{k=1}^{N} \sigma(\bm{W}_k^{(0)}x +\bm{b}_k^{0})\sigma(\bm{W}_k^{(0)}x' + \bm{b}_k^{(0)}) \\
    & \overset{\mathcal{P}}{\longrightarrow} \mathbb{E} \Big[  \sigma(\bm{W}_k^{(0)}x +\bm{b}_k^{(0)})\sigma(\bm{W}_k^{(0)}x' +\bm{b}_k^{(0)}) \Big] = \Sigma^{(1)}(x, x'),
\end{align*}
as ${N} \rightarrow \infty$, where $\Sigma^{(1)}(x, x') $ is defined in equation \ref{eq: NN_Gaussian_cov}.

Also,
\begin{align*}
       \sum_{k=1}^{N} \frac{\partial u(x, \bm{\theta})}{\partial \bm{b}^{(0)}_k}  \frac{\partial u(x', \bm{\theta})}{\partial \bm{b}^{(0)}_k}  &= \frac{1}{N}\sum_{k=1}^{N}  \Big[ \big(  \bm{W}^{(1)}_k  \big)^2 \Dot{\sigma}(\bm{W}_k^{(0)} x + \bm{b}^{(0)}_k)  \Dot{\sigma}(\bm{W}_k^{(0)} x' + \bm{b}^{(0)}_k) \Big] \\
       &\overset{\mathcal{P}}{\longrightarrow} \mathbb{E} \Big[  \Dot{\sigma} (\bm{W}_k^{(0)}x +  \bm{b}_k^{(0)})  \Dot{\sigma} (\bm{W}_k^{(0)}x' +  \bm{b}_k^{(0)})   \Big]= \Dot{\Sigma}^{(1)}(x,x').
\end{align*}
Then plugging all these together we obtain
\begin{align*}
     \bm{K}_{uu}(0) &=  \Big\langle  \frac{d u(x, {\bm \theta}(0))}{d{\bm \theta}},  \frac{d u(x', {\bm \theta}(0))}{d{\bm \theta}}    \Big\rangle \\
     &= \sum_{l=0}^{1}  \sum_{k=1}^{N}  \frac{\partial u(x, \bm{\theta})}{\partial \bm{W}^{(l)}_k}  \frac{\partial u(x', \bm{\theta})}{\partial \bm{W}^{(l)}_k} + \sum_{k=1}^{N} \frac{\partial u(x, \bm{\theta})}{\partial \bm{b}^{(0)}_k}  \frac{\partial u(x', \bm{\theta})}{\partial \bm{b}^{(0)}_k} + \frac{\partial u(x, \bm{\theta})}{\partial \bm{b}^{(1)}}  \frac{\partial u(x', \bm{\theta})}{\partial \bm{b}^{(1)}} \\
     &  \overset{\mathcal{P}}{\longrightarrow} \Dot{\Sigma}^{(1)}(x,x')(xx') + \Sigma^{(1)}(x, x')  + \Dot{\Sigma}^{(1)}(x,x') + 1 \triangleq \Theta_{uu}^{(1)},
\end{align*}
as $N \rightarrow \infty$. This formula is also consistent with equation \ref{eq: NTK}.

Next, we compute $\bm{K}_{rr}(0)$.
To this end, recall that
\begin{align*}
         u_{xx}(x, \bm{\theta}) = \frac{1}{\sqrt{N}} \bm{W}^{(1)} \cdot \Big[  \Ddot{\sigma}(\bm{W}^{(0)}x + \bm{b}^{(0)} ) \odot \bm{W}^{(0)}  \odot \bm{W}^{(0)} \Big] =  \frac{1}{\sqrt{N}} \sum_{k=1}^N \bm{W}^{(1)}_k (\bm{W}^{(0)}_k)^2
    \Ddot{\sigma}( \bm{W}^{(0)}_kx +  \bm{b}^{(0)}_k ).
\end{align*}
It is then easy to compute that
\begin{align*}
     &\frac{\partial u_{xx}(x, \bm{\theta})}{\partial \bm{W}^{(0)}_k}= \frac{1}{\sqrt{N}} \bm{W}^{(1)}_k \bm{W}^{(0)}_k  \Big[\bm{W}_k^{(0)} \dddot{\sigma}(\bm{W}_k^{(0)} x + \bm{b}^{(0)}_k ) x + 2 \ddot{\sigma}(\bm{W}_k^{(0)} x + \bm{b}^{(0)}_k) \Big]
    \\
     &\frac{\partial u_{xx}(x, \bm{\theta})}{\partial \bm{W}^{(1)}_k}=  \frac{1}{\sqrt{N}}  (\bm{W}_k^{(0)})^2   \Ddot{\sigma}(\bm{W}_k^{(0)} x + \bm{b}^{(0)}_k)\\
     & \frac{\partial u_{xx}(x, \bm{\theta})}{\partial \bm{b}^{(0)}_k} =  \frac{1}{\sqrt{N}}  \bm{W}^{(1)}_k (\bm{W}^{(0)}_k)^2
    \dddot{\sigma}( \bm{W}^{(0)}_kx +  \bm{b}^{(0)}_k ),
\end{align*}
where $ \dddot{\sigma}$ denotes third order derivative of $\sigma$. Then we have
\begin{align*}
      \sum_{k=1}^{N}  \frac{\partial u_{xx}(x, \bm{\theta})}{\partial \bm{W}^{(0)}_k}  \frac{\partial u_{xx}(x', \bm{\theta})}{\partial \bm{W}^{(0)}_k}
        &= \frac{1}{N}  \sum_{k=1}^{N} \Big(\bm{W}^{(1)}_k \bm{W}^{(0)}_k \Big)^2  \bigg(  \Big[\bm{W}_k^{(0)} \dddot{\sigma}(\bm{W}_k^{(0)} x + \bm{b}^{(0)}_k ) x + 2 \ddot{\sigma}(\bm{W}_k^{(0)} x + \bm{b}^{(0)}_k) \Big]       \bigg)\\
        &\cdot
       \bigg(  \Big[\bm{W}_k^{(0)} \dddot{\sigma}(\bm{W}_k^{(0)} x' + \bm{b}^{(0)}_k ) x' + 2 \ddot{\sigma}(\bm{W}_k^{(0)} x' + \bm{b}^{(0)}_k) \Big]       \bigg) \\
       &= I_1 + I_2 + I_3 + I_4,
\end{align*}
where
\begin{align*}
    &I_1 =  \frac{1}{N} \sum_{k=1}^{N} \big(\bm{W}^{(1)}_k  \big)^2 \big( \bm{W}^{(0)}_k \big)^4 \dddot{\sigma}(\bm{W}_k^{(0)} x + \bm{b}^{(0)}_k ) x \cdot \dddot{\sigma}(\bm{W}_k^{(0)} x' + \bm{b}^{(0)}_k ) x' \\
    & I_2 =  \frac{2}{N} \sum_{k=1}^{N} \Big(\bm{W}^{(1)}_k  \Big)^2  \Big( \bm{W}^{(0)}_k \Big)^3 \dddot{\sigma}(\bm{W}_k^{(0)} x + \bm{b}^{(0)}_k )  \cdot \ddot{\sigma}(\bm{W}_k^{(0)} x' + \bm{b}^{(0)}_k) x \\
    & I_3 =  \frac{2}{N} \sum_{k=1}^{N} \Big(\bm{W}^{(1)}_k  \Big)^2  \Big( \bm{W}^{(0)}_k \Big)^3 \dddot{\sigma}(\bm{W}_k^{(0)} x' + \bm{b}^{(0)}_k )  \cdot \ddot{\sigma}(\bm{W}_k^{(0)} x + \bm{b}^{(0)}_k)x' \\
    & I_4  =  \frac{4}{N} \sum_{k=1}^{N}  \Big(\bm{W}^{(1)}_k \bm{W}^{(0)}_k \Big)^2  \ddot{\sigma}(\bm{W}_k^{(0)} x + \bm{b}^{(0)}_k) \cdot  \ddot{\sigma}(\bm{W}_k^{(0)} x' + \bm{b}^{(0)}_k).
\end{align*}
By the law of large numbers, letting $N \rightarrow \infty$ gives
\begin{align*}
    & I_1 \overset{\mathcal{P}}{\longrightarrow} \mathbb{E} \Big[ \big( \bm{W}^{(0)}_k \big)^4 \dddot{\sigma}(\bm{W}_k^{(0)} x + \bm{b}^{(0)}_k )  \cdot \dddot{\sigma}(\bm{W}_k^{(0)} x' + \bm{b}^{(0)}_k ) \Big] x x' : = J_1 \\
    & I_2  \overset{\mathcal{P}}{\longrightarrow}  \mathbb{E} \Big[ \big( \bm{W}^{(0)}_k \big)^3  \dddot{\sigma}(\bm{W}_k^{(0)} x + \bm{b}^{(0)}_k )  \cdot \ddot{\sigma}(\bm{W}_k^{(0)} x' + \bm{b}^{(0)}_k)    \Big] x   : = J_2  \\
    & I_3  \overset{\mathcal{P}}{\longrightarrow}  \mathbb{E} \Big[ \big( \bm{W}^{(0)}_k \big)^3  \dddot{\sigma}(\bm{W}_k^{(0)} x' + \bm{b}^{(0)}_k )  \cdot \ddot{\sigma}(\bm{W}_k^{(0)} x + \bm{b}^{(0)}_k)    \Big] x'  : = J_3  \\
    & I_4  \overset{\mathcal{P}}{\longrightarrow}  \mathbb{E} \Big[
    \big( \bm{W}^{(0)}_k \big)^2   \ddot{\sigma}(\bm{W}_k^{(0)} x + \bm{b}^{(0)}_k) \cdot  \ddot{\sigma}(\bm{W}_k^{(0)} x' + \bm{b}^{(0)}_k)  \Big] : = J_4.
\end{align*}
In conclusion we have
\begin{align*}
     \sum_{k=1}^{N}  \frac{\partial u_{xx}(x, \bm{\theta})}{\partial \bm{W}^{(0)}_k}  \frac{\partial u_{xx}(x', \bm{\theta})}{\partial \bm{W}^{(0)}_k}  \overset{\mathcal{P}}{\longrightarrow}  J_1 + J_2 + J_3 + J_4  : = A_{rr}.
\end{align*}
Moreover,
\begin{align*}
      \sum_{k=1}^{N}  \frac{\partial u_{xx}(x, \bm{\theta})}{\partial \bm{W}^{(1)}_k}  \frac{\partial u_{xx}(x', \bm{\theta})}{\partial \bm{W}^{(1)}_k} &=  \frac{1}{N} \sum_{k=1}^{N}
      \big(\bm{W}_k^{(0)} \big)^4   \Ddot{\sigma}(\bm{W}_k^{(0)} x + \bm{b}^{(0)}_k)  \cdot \Ddot{\sigma}(\bm{W}_k^{(0)} x + \bm{b}^{(0)}_k) \\
      &  \overset{\mathcal{P}}{\longrightarrow} \mathbb{E} \Big[      \big(\bm{W}_k^{(0)} \big)^4   \Ddot{\sigma}(\bm{W}_k^{(0)} x + \bm{b}^{(0)}_k)  \cdot \Ddot{\sigma}(\bm{W}_k^{(0)} x + \bm{b}^{(0)}_k)  \Big] : = B_{rr},
\end{align*}
and
\begin{align*}
     \sum_{k=1}^{N}  \frac{\partial u_{xx}(x, \bm{\theta})}{\partial \bm{b}^{(0)}_k}  \frac{\partial u_{xx}(x', \bm{\theta})}{\partial \bm{b}^{(0)}_k} &=  \frac{1}{N}  \sum_{k=1}^{N}
     (\bm{W}^{(1)}_k)^2 (\bm{W}^{(0)}_k)^4 \Big[
    \dddot{\sigma}( \bm{W}^{(0)}_k x +  \bm{b}^{(0)}_k ) \cdot
    \dddot{\sigma}( \bm{W}^{(0)}_k x' +  \bm{b}^{(0)}_k ) \Big] \\
    &  \overset{\mathcal{P}}{\longrightarrow} \mathbb{E} \Big[   (\bm{W}^{(0)}_k)^4 \big(
    \dddot{\sigma}( \bm{W}^{(0)}_k x +  \bm{b}^{(0)}_k ) \cdot
    \dddot{\sigma}( \bm{W}^{(0)}_k x' +  \bm{b}^{(0)}_k ) \big)   \Big] : = C_{rr}.
\end{align*}
Now, recall that
\begin{align*}
    \bm{K}_{rr}(0) = \Big \langle \frac{d u_{xx}(\bm{x}, {\bm \theta}(0))}{d{\bm \theta}},  \frac{d u_{xx}(\bm{x}', {\bm \theta}(0))}{d{\bm \theta}}   \Big\rangle.
\end{align*}
Thus we can conclude that as $N \rightarrow \infty$,
\begin{align*}
     \bm{K}_{rr}(0) \overset{\mathcal{P}}{\longrightarrow}  A_{rr} + B_{rr} + C_{rr} : = \Theta_{rr}(x,x').
\end{align*}
Finally, recall that $\bm{K}_{ur}(x, x') = \bm{K}_{ru}(x', x)$. So it suffices to compute $\bm{K}_{ur}(x, x')$ and its limit. To this end, recall that
\begin{align*}
     \bm{K}_{ur}(x, x') = \Big\langle  \frac{d u(\bm{x}, {\bm \theta}(t))}{d{\bm \theta}} , \frac{d u_{xx}(\bm{x}', {\bm \theta}(t))}{d{\bm \theta}}  \Big\rangle.
\end{align*}
Then letting $N \rightarrow \infty$ gives
\begin{align*}
    & \sum_{k=1}^{N}  \frac{\partial u(x, \bm{\theta})}{\partial \bm{W}^{(0)}_k}  \frac{\partial u_{xx}(x', \bm{\theta})}{\partial \bm{W}^{(0)}_k}
        \\
        &= \frac{1}{N}  \sum_{k=1}^{N} \big(  \bm{W}^{(1)}_k  \big)^2  \bm{W}^{(0)}_k   \Big[
    \Dot{\sigma}(\bm{W}_k^{(0)} x + \bm{b}^{(0)}_k)x  \Big] \cdot
    \Big[ \bm{W}^{(0)}_k   \dddot{\sigma}(\bm{W}_k^{(0)} x' + \bm{b}^{(0)}_k ) x' + 2 \ddot{\sigma}(\bm{W}_k^{(0)} x' + \bm{b}^{(0)}_k) \Big] \\
    &  \overset{\mathcal{P}}{\longrightarrow} \mathbb{E} \Big[ \big(   \bm{W}^{(0)}_k \big)^2      \Dot{\sigma}(\bm{W}_k^{(0)} x + \bm{b}^{(0)}_k)\cdot  \dddot{\sigma}(\bm{W}_k^{(0)} x' + \bm{b}^{(0)}_k )    \Big] (xx')
    + 2 \mathbb{E} \Big[  \bm{W}^{(0)}_k   \Dot{\sigma}(\bm{W}_k^{(0)} x + \bm{b}^{(0)}_k)\cdot \ddot{\sigma}(\bm{W}_k^{(0)} x' + \bm{b}^{(0)}_k)     \Big] x \\
    & : = A_{ur},
\end{align*}
and
\begin{align*}
   \sum_{k=1}^{N}  \frac{\partial u(x, \bm{\theta})}{\partial \bm{W}^{(1)}_k}  \frac{\partial u_{xx}(x', \bm{\theta})}{\partial \bm{W}^{(1)}_k} &= \frac{1}{N}  \sum_{k=1}^{N}  \Big[  \big(\bm{W}_k^{(0)} \big)^2  \sigma(\bm{W}_k^{(0)} x + \bm{b}^{(0)}_k) \cdot     \Ddot{\sigma}(\bm{W}_k^{(0)} x + \bm{b}^{(0)}_k)  \Big] \\
   & \overset{\mathcal{P}}{\longrightarrow} \mathbb{E} \Big[
   \big(\bm{W}_k^{(0)} \big)^2  \sigma(\bm{W}_k^{(0)} x + \bm{b}^{(0)}_k) \cdot     \Ddot{\sigma}(\bm{W}_k^{(0)} x + \bm{b}^{(0)}_k)
   \Big] : = B_{ur},
\end{align*}
and
\begin{align*}
     \sum_{k=1}^{N}  \frac{\partial u(x, \bm{\theta})}{\partial \bm{b}^{(0)}_k}  \frac{\partial u_{xx}(x', \bm{\theta})}{\partial \bm{b}^{(0)}_k}& =  \frac{1}{N}  \sum_{k=1}^{N} \Big[\bm{W}^{(1)}_k  \bm{W}^{(0)}_k \Big]^2 \cdot   \Big[
        \Dot{\sigma}(\bm{W}_k^{(0)} x + \bm{b}^{(0)}_k) \cdot
    \dddot{\sigma}( \bm{W}^{(0)}_kx' +  \bm{b}^{(0)}_k )
     \Big] \\
     & \overset{\mathcal{P}}{\longrightarrow} \mathbb{E} \Big[
     \big( \bm{W}^{(0)}_k \big)^2 \cdot \big(
     \Dot{\sigma}(\bm{W}_k^{(0)} x + \bm{b}^{(0)}_k) \cdot
    \dddot{\sigma}( \bm{W}^{(0)}_kx' +  \bm{b}^{(0)}_k )
     \big)^2
     \Big] : = C_{ur}.
\end{align*}
As a result, we obtain
\begin{align*}
     \bm{K}_{ur}(x, x') = \Big\langle  \frac{d u(\bm{x}, {\bm \theta}(t))}{d{\bm \theta}} , \frac{d u_{xx}(\bm{x}', {\bm \theta}(t))}{d{\bm \theta}}  \Big\rangle \overset{\mathcal{P}}{\longrightarrow} A_{ur} + B_{ur} + C_{ur} : \Theta_{ur}^{(1)},
\end{align*}
as $N \rightarrow \infty$. This concludes the proof.

\end{proof}

\section{Proof of Theorem \ref{theorem: kernel_constant}}
\label{sec: proof_kernel_constant}
Before we prove the main theorem, we need to prove a series of lemmas.

\begin{lemma}
\label{lemma: output_grad_bounded}
Under the setting of Theorem \ref{theorem: kernel_constant}, for $  l=0,1$, we have
\begin{align*}
    & \sup_{t \in [0,T]}  \left\| \frac{\partial u}{\partial \bm{W}^{(l)} } \right\|_\infty = \mathcal{O}(\frac{1}{\sqrt{N}}), \\
    & \sup_{t \in [0,T]}  \left\| \frac{\partial u_{xx}}{\partial \bm{W}^{(l)} } \right\|_\infty =  \mathcal{O}(\frac{1}{\sqrt{N}}),\\
    &  \sup_{t \in [0,T]}  \left\| \frac{\partial u}{\partial \bm{b}^{(0)} } \right\|_\infty = \mathcal{O}(\frac{1}{\sqrt{N}}),\\
    &\sup_{t \in [0,T]}  \left\| \frac{\partial u_{xx}}{\partial \bm{b}^{(0)} } \right\|_\infty = \mathcal{O}(\frac{1}{\sqrt{N}}).
\end{align*}

\end{lemma}

\begin{proof}
For the given model problem, recall that
\begin{align*}
     u(x, \bm{\theta}) =   \frac{1}{\sqrt{N}}  \bm{W}^{(1)}\sigma(\bm{W}^{(0)}(t)x + \bm{b}^{(0)}) + \bm{b}^{(1)},
\end{align*}
and
\begin{align*}
     &\frac{\partial u(x, \bm{\theta})}{\partial \bm{W}^{(0)}_k}=  \frac{1}{\sqrt{N}} \bm{W}^{(1)}_k
    \Dot{\sigma}(\bm{W}_k^{(0)} x + \bm{b}^{(0)}_k)x
    \\
     &\frac{\partial u(x, \bm{\theta})}{\partial \bm{W}^{(1)}_k}=  \frac{1}{\sqrt{N}} \sigma(\bm{W}_k^{(0)} x + \bm{b}^{(0)}_k) \\
     & \frac{\partial u(x, \bm{\theta})}{\partial \bm{b}^{(0)}_k} =  \frac{1}{\sqrt{N}}   \bm{W}^{(1)}_k  \Dot{\sigma}(\bm{W}_k^{(0)} x + \bm{b}^{(0)}_k).
\end{align*}
Then by assumptions (i), (ii),
and given that $\Omega$ is bounded, we have
\begin{align*}
   &\sup_{t \in [0,T]}  \left\| \frac{\partial u}{\partial \bm{W}^{(l)} } \right\|_\infty \leq  \frac{C}{\sqrt{N}}, \quad  l=0,1. \\
   &\sup_{t \in [0,T]}  \left\| \frac{\partial u}{\partial \bm{b}^{(0)} } \right\|_\infty \leq  \frac{C}{\sqrt{N}}.
\end{align*}
Also,
\begin{align*}
         u_{xx}(x, \bm{\theta}) = \frac{1}{\sqrt{N}} \bm{W}^{(1)} \cdot \Big[  \Ddot{\sigma}(\bm{W}^{(0)}x + \bm{b}^{(0)} ) \odot \bm{W}^{(0)}  \odot \bm{W}^{(0)} \Big] =  \frac{1}{\sqrt{N}} \sum_{k=1}^N \bm{W}^{(1)}_k (\bm{W}^{(0)}_k)^2
    \Ddot{\sigma}( \bm{W}^{(0)}_kx +  \bm{b}^{(0)}_k ),
\end{align*}
and
\begin{align*}
     &\frac{\partial u_{xx}(x, \bm{\theta})}{\partial \bm{W}^{(0)}_k}= \frac{1}{\sqrt{N}} \bm{W}^{(1)}_k \bm{W}^{(0)}_k  \Big[\bm{W}_k^{(0)} \dddot{\sigma}(\bm{W}_k^{(0)} x + \bm{b}^{(0)}_k ) x + 2 \ddot{\sigma}(\bm{W}_k^{(0)} x + \bm{b}^{(0)}_k) \Big]
    \\
     &\frac{\partial u_{xx}(x, \bm{\theta})}{\partial \bm{W}^{(1)}_k}=  \frac{1}{\sqrt{N}}  (\bm{W}_k^{(0)})^2   \Ddot{\sigma}(\bm{W}_k^{(0)} x + \bm{b}^{(0)}_k)\\
     & \frac{\partial u_{xx}(x, \bm{\theta})}{\partial \bm{b}^{(0)}_k} =  \frac{1}{\sqrt{N}}  \bm{W}^{(1)}_k (\bm{W}^{(0)}_k)^2
    \dddot{\sigma}( \bm{W}^{(0)}_kx +  \bm{b}^{(0)}_k ).
\end{align*}
Again, using assumptions (i), (ii) gives
\begin{align*}
     &\sup_{t \in [0,T]}  \left\| \frac{\partial u_{xx}}{\partial \bm{W}^{(l)} } \right\|_\infty \leq  \frac{C^4}{\sqrt{N}}, \quad  l=0,1. \\
     &\sup_{t \in [0,T]}  \left\| \frac{\partial u_{xx}}{\partial \bm{b}^{(0)} } \right\|_\infty \leq  \frac{C^4}{\sqrt{N}}.
\end{align*}
This completes the proof.
\end{proof}

\begin{lemma}
\label{lemma: weight_change_little}
Under the setting of Theorem \ref{theorem: kernel_constant}, we have
\begin{align}
    &\lim_{N \rightarrow \infty} \sup_{t \in [0,T]} \Big\| \frac{1}{\sqrt{N}}  \Big(\bm{W}^{(l)}(t) - \bm{W}^{(l)}(0)   \Big) \Big\|_2 =0, \quad l = 0,1. \\
   &   \lim_{N \rightarrow \infty} \sup_{t \in [0,T]} \Big\| \frac{1}{\sqrt{N}}  \Big(\bm{b}^{(0)}(t) - \bm{b}^{(0)}(0)   \Big) \Big\|_2 =0.
\end{align}

\end{lemma}

\begin{proof}
Recall that the loss function for the model problem \ref{eq: model problem} is given by
\begin{align*}
      \mathcal{L}(\bm{\theta}) = \mathcal{L}_b(\bm{\theta}) + \mathcal{L}_r(\bm{\theta}) = \frac{1}{2} \sum_{i=1}^{N_b}  |u(x_b^i, \bm{\theta}) - g(x_b^i)   |^2  +  \frac{1}{2} \sum_{i=1}^{N_r} |u_{xx}(x_r^i, \bm{\theta}) - f(x_r^i)|^2.
\end{align*}
Consider minimizing the loss function $\mathcal{L}(\bm{\theta})$
by gradient descent with an infinitesimally
small learning rate:
\begin{align*}
     \frac{d \bm{\theta}}{dt} = - \nabla \mathcal{L}(\bm{\theta}).
\end{align*}
This implies that
\begin{align*}
    &\frac{d \bm{W}^{(l)}}{d t} = - \frac{\partial \mathcal{L}(\bm{\theta})}{\partial  \bm{W}^{(l)}},  \quad l=0, 1,\\
  & \frac{d \bm{b}^{(0)}}{d t} = - \frac{\partial \mathcal{L}(\bm{\theta})}{\partial  \bm{b}^{(0)}}.
\end{align*}
 Then  we have
\begin{align*}
   & \left\|\frac{1}{\sqrt{N}}\left(\bm{W}^{(l)}(t)-\bm{W}^{(l)}(0)\right)\right\|_{2}  = \left\|\frac{1}{\sqrt{N}} \int_0^t \frac{d \bm{W}^{(l)}(\tau)}{d \tau} d\tau    \right\|_{2} = \left\|\frac{1}{\sqrt{N}} \int_0^t \frac{\partial \mathcal{L}(\bm{\theta}(\tau))}{\partial \bm{W}^{(l)}}    d\tau    \right\|_{2} \\
    &=  \left\|\frac{1}{\sqrt{N}} \int_0^t \Big[\sum_{i = 1}^{N_b} (u(x_b^i,\bm{\theta}(\tau)) -g(x_b^i)) \frac{\partial u}{\partial \bm{W}^{(l)} }(x_b^i,{\bm \theta}(\tau))  +\sum_{i = 1}^{N_r} (u_{xx}(x_r^i,{\bm \theta}(\tau)) -f(x_r^i))\frac{\partial u_{xx}}{\partial {\bm W}^{(l)}}(x_r^i,{\bm \theta}(\tau))    \Big]            d\tau    \right\|_{2} \\
    & \leq I_1^{(l)} + I_2^{(l)},
\end{align*}
where
\begin{align*}
     &I_1^{(l)} = \left\|\frac{1}{\sqrt{N}} \int_0^t \Big[\sum_{i = 1}^{N_b} (u(x_b^i,\bm{\theta}(\tau)) -g(x_b^i)) \frac{\partial u}{\partial \bm{W}^{(l)} }(x_b^i,{\bm \theta}(\tau)) \Big]            d\tau    \right\|_{2}     \\
    & I_2^{(l)} = \left\|\frac{1}{\sqrt{N}} \int_0^t \Big[ \sum_{i = 1}^{N_r} (u_{xx}(x_r^i,{\bm \theta}(\tau)) -f(x_r^i))\frac{\partial u_{xx}}{\partial {\bm W^{(l)}}}(x_r^i,{\bm \theta}(\tau))    \Big]            d\tau    \right\|_{2}.
\end{align*}

We first process to estimate $I_1^{(l)}$ as
\begin{align*}
    I_1^{(l)} &\leq  \frac{1}{\sqrt{N}} \bigintsss_0^t  \left\| \Big[\sum_{i = 1}^{N_b} (u(x_b^i,\bm{\theta}(\tau)) -g(x_b^i)) \frac{\partial u}{\partial \bm{W}^{(l)} }(x_b^i,{\bm \theta}(\tau)) \Big]   \right\|_{2}             d\tau   \\
    & =  \frac{1}{\sqrt{N}} \bigintsss_0^t  \sqrt{ \sum_{k=1}^N  \Big(\sum_{i = 1}^{N_b} \big(u(x_b^i,\bm{\theta}(\tau)) -g(x_b^i) \big) \frac{\partial u}{\partial \bm{W}_k^{(l)} }(x_b^i,{\bm \theta}(\tau))       \Big)^2  }   d\tau \\
    &\leq  \frac{1}{\sqrt{N}} \bigintsss_0^T \left\| \frac{\partial u}{\partial \bm{W}^{(l)} }(x_b^i,{\bm \theta}(\tau)) \right\|_\infty  \sqrt{ \sum_{k=1}^N  \Big(\sum_{i = 1}^{N_b} \big(u(x_b^i,\bm{\theta}(\tau)) -g(x_b^i) \big)       \Big)^2  } d \tau \\
    &=  \frac{1}{\sqrt{N}} \bigintsss_0^T \sqrt{N}   \left\| \frac{\partial u}{\partial \bm{W}^{(l)} }(x_b^i,{\bm \theta}(\tau)) \right\|_\infty \cdot  \Big| \sum_{i = 1}^{N_b} \big(u(x_b^i,\bm{\theta}(\tau)) -g(x_b^i) \big)  \Big| d \tau.
\end{align*}
Thus, by assumptions and Lemma \ref{lemma: output_grad_bounded}, for $l=0,1$  we have
\begin{align*}
    \sup_{t \in [0,T]} I_1^{(l)} & = \sup_{t \in [0,T]}  \bigintsss_0^T  \left\| \frac{\partial u}{\partial \bm{W}^{(l)} }(x_b^i,{\bm \theta}(\tau)) \right\|_\infty \cdot  \Big| \sum_{i = 1}^{N_b} \big(u(x_b^i,\bm{\theta}(\tau)) -g(x_b^i) \big)  \Big| d \tau \\
    &\leq \frac{C}{\sqrt{N}}  \longrightarrow 0, \quad \text{ as } N\longrightarrow \infty.
\end{align*}
Similarly,
\begin{align*}
     \sup_{t \in [0,T]} I_2^{(l)} \leq  \sup_{t \in [0,T]}  &\leq  \frac{1}{\sqrt{N}} \bigintsss_0^T \left\| \frac{\partial u}{\partial \bm{W}^{(l)} }(x_b^i,{\bm \theta}(\tau)) \right\|_\infty  \sqrt{ \sum_{k=1}^N  \Big(\sum_{i = 1}^{N_b} \big(u(x_b^i,\bm{\theta}(\tau)) -g(x_b^i) \big)       \Big)^2  } d \tau \\
    &=  \frac{1}{\sqrt{N}} \bigintsss_0^T \sqrt{N}   \left\| \frac{\partial u_{xx}}{\partial \bm{W}^{(l)} }(x_r^i,{\bm \theta}(\tau)) \right\|_\infty \cdot  \Big| \sum_{i = 1}^{N_r} \big(u_{xx}(x_r^i,\bm{\theta}(\tau)) -f(x_r^i) \big)  \Big| d \tau  \\
    &\leq \frac{C^4}{\sqrt{N}} \longrightarrow 0, \quad \text{ as } N\longrightarrow \infty.
\end{align*}
Plugging these together, we obtain
\begin{align*}
    \lim_{N \rightarrow \infty} \sup_{t \in  [0,T]} \left\|\frac{1}{\sqrt{N}}\left(\bm{W}^{(l)}(t)-\bm{W}^{(l)}(0)\right)\right\|_{2} \leq \lim_{N \rightarrow \infty} \sup_{t \in  [0,T]}  I_1^{(l)} +  I_2^{(l)} = 0,
\end{align*}
for $l=1,2$. Similarly, applying the same strategy to $\bm{b}^{(0)}$ we can show
\begin{align*}
     \lim_{N \rightarrow \infty} \sup_{t \in [0,T]} \Big\| \frac{1}{\sqrt{N}}  \Big(\bm{b}^{(0)}(t) - \bm{b}^{(0)}(0)   \Big) \Big\|_2 =0.
\end{align*}
This concludes the proof.
\end{proof}

\begin{lemma}
\label{lemma: sigma_change_little}
Under the setting of Theorem \ref{theorem: kernel_constant}, we have
\begin{align}
    \label{eq: sigma_change_little}
     \lim_{N \rightarrow \infty} \sup_{t \in [0,T]} \Big\| \frac{1}{\sqrt{N}} \Big( {\sigma}^{(k)}(\bm{W}^{(0)}(t)x + \bm{b}^{(0)}(t)) - {\sigma}^{(k)}(\bm{W}^{(0)}(t)x + \bm{b}^{(0)}(0)) \Big) \Big\|_2 = 0,
\end{align}
for $k=0,1,2,3$, where $\sigma^{(k)}$ denotes the $k$-th order derivative of $\sigma$.
\end{lemma}

\begin{proof}
By the mean-value theorem for vector-valued function and Lemma \ref{lemma: weight_change_little}, there exists $\xi$
\begin{align*}
     &\left\|  \frac{1}{\sqrt{N}} \Big( {\sigma}^{(k)}(\bm{W}^{(0)}(t)x + \bm{b}^{(0)}(t)) - {\sigma}^{(k)}(\bm{W}^{(0)}(0)x + \bm{b}^{(0)}(0)) \Big)  \right\|_2 \\
     &\leq \left\| \sigma^{(k+1)}(\xi) \right\|  \left\| \frac{1}{\sqrt{N}}      \Big(\bm{W}^{(0)}(t)x + \bm{b}^{(0)}(t) - \bm{W}^{(0)}(0)x + \bm{b}^{(0)}(0)                   \Big)    \right\|_2 \\
     & \leq C\left\| \frac{1}{\sqrt{N}}    \left(\bm{W}^{(0)}(t) -\bm{W}^{(0)}(0)\right)   \right\|_2 + C\left\| \frac{1}{\sqrt{N}}    \left(\bm{b}^{(0)}(t) -\bm{b}^{(0)}(0)\right)   \right\|_2 \\
     & \longrightarrow 0,
\end{align*}
as $N \rightarrow \infty$. Here we use the assumption that $\sigma^{(k)}$ is bounded for $k=0,1,2,3,4$. This concludes the proof.
\end{proof}

\begin{lemma}
\label{lemma: output_grad_constant}
Under the setting of Theorem \ref{theorem: kernel_constant}, we have
\begin{align}
 & \lim_{N \rightarrow \infty} \sup_{t \in [0,T]}  \left\|   {\frac{\partial u(x, \bm{\theta}(t))}{\partial \bm{\theta}}} -  {\frac{\partial u(x, \bm{\theta}(0))}{\partial \bm{\theta}}}      \right\|_2 \\
  &\lim_{N \rightarrow \infty} \sup_{t \in [0,T]}  \left\|   {\frac{\partial u_{xx}(x, \bm{\theta}(t))}{\partial \bm{\theta}}} -  {\frac{\partial u_{xx}(x, \bm{\theta}(0))}{\partial \bm{\theta}}}      \right\|_2.
\end{align}
\end{lemma}

\begin{proof}
Recall that
\begin{align*}
     &\frac{\partial u(x, \bm{\theta})}{\partial \bm{W}^{(0)}_k}=  \frac{1}{\sqrt{N}} \bm{W}^{(1)}_k
    \Dot{\sigma}(\bm{W}_k^{(0)} x + \bm{b}^{(0)}_k)x
    \\
     &\frac{\partial u(x, \bm{\theta})}{\partial \bm{W}^{(1)}_k}=  \frac{1}{\sqrt{N}} \sigma(\bm{W}_k^{(0)} x + \bm{b}^{(0)}_k) \\
     & \frac{\partial u(x, \bm{\theta})}{\partial \bm{b}^{(0)}_k} =  \frac{1}{\sqrt{N}}   \bm{W}^{(1)}_k  \Dot{\sigma}(\bm{W}_k^{(0)} x + \bm{b}^{(0)}_k) \\
     &\frac{\partial u(x, \bm{\theta})}{\partial \bm{b}^{(1)}} = 1.
\end{align*}
To simplify notation, let us define
\begin{align*}
    &\bm{A}(t) = [\bm{W}^{(1)}(t)]^T \\
    &\bm{B}(t) = \Dot{\sigma}(\bm{W}^{(0)}(t) x + \bm{b}^{(0)}(t))x.
\end{align*}
Then by assumption (i)  Lemma \ref{lemma: weight_change_little} \ref{lemma: sigma_change_little},  we have
\begin{align*}
   & \sup_{t \in [0, T]}  \left\|\frac{\partial u(x, \bm{\theta}(t))}{\partial \bm{W}^{(0)}} - \frac{\partial u(x, \bm{\theta}(0))}{\partial \bm{W}^{(0)}}      \right\|_2  \\
   & =  \sup_{t \in [0, T]} \left\|\frac{1}{\sqrt{N}}  \Big(  \bm{A}(t) \odot
   \bm{B}(t)  - \bm{A}(0) \odot \bm{B}(0) \Big) \right\|_2 \\
    & \leq  \sup_{t \in [0, T]}\left\|\frac{1}{\sqrt{N}}  \Big( \bm{A}(t) -\bm{A}(0) \Big) \odot \bm{B}(t) \right\|_2
    + \left\|\frac{1}{\sqrt{N}} \bm{A}(0)  \odot \Big( \bm{B}(t) -\bm{B}(0) \Big)  \right\|_2 \\
    & \leq  \sup_{t \in [0, T]}\left\|\bm{B}(t) \right\|_\infty \left\| \frac{1}{\sqrt{N}}  \Big( \bm{A}(t) -\bm{A}(0) \Big) \right\|_2  +   \sup_{t \in [0, T]} \left\|\bm{A}(0) \right\|_\infty \left\| \frac{1}{\sqrt{N}}  \Big( \bm{B}(t) -\bm{B}(0) \Big) \right\|_2 \\
    &\longrightarrow 0,
\end{align*}
as $N \rightarrow \infty$. Here $\odot$ denotes point-wise multiplication.

Similarly, we can show that
\begin{align*}
   & \lim_{N \rightarrow \infty} \sup_{t \in [0, T]} \left\|\frac{\partial u(x, \bm{\theta}(t))}{\partial \bm{W}^{(1)}} - \frac{\partial u(x, \bm{\theta}(0))}{\partial \bm{W}^{(1)}}      \right\|_2  = 0,\\
  & \lim_{N \rightarrow \infty} \sup_{t \in [0, T]} \left\|\frac{\partial u(x, \bm{\theta}(t))}{\partial \bm{b}^{(0)}} - \frac{\partial u(x, \bm{\theta}(0))}{\partial \bm{b}^{(0)}}  \right\|_2  = 0.
\end{align*}
Thus, we conclude that
\begin{align*}
     \lim_{N \rightarrow \infty} \sup_{t \in [0,T]}  \left\|   {\frac{\partial u(x, \bm{\theta}(t))}{\partial \bm{\theta}}} -  {\frac{\partial u(x, \bm{\theta}(0))}{\partial \bm{\theta}}}      \right\|_2 = 0.
\end{align*}

Now for $u_{xx}$, we know that
\begin{align*}
      &\frac{\partial u_{xx}(x, \bm{\theta})}{\partial \bm{W}^{(0)}_k}= \frac{1}{\sqrt{N}} \bm{W}^{(1)}_k \bm{W}^{(0)}_k  \Big[\bm{W}_k^{(0)} \dddot{\sigma}(\bm{W}_k^{(0)} x + \bm{b}^{(0)}_k ) x + 2 \ddot{\sigma}(\bm{W}_k^{(0)} x + \bm{b}^{(0)}_k) \Big]
    \\
     &\frac{\partial u_{xx}(x, \bm{\theta})}{\partial \bm{W}^{(1)}_k}=  \frac{1}{\sqrt{N}}  (\bm{W}_k^{(0)})^2   \Ddot{\sigma}(\bm{W}_k^{(0)} x + \bm{b}^{(0)}_k)\\
     & \frac{\partial u_{xx}(x, \bm{\theta})}{\partial \bm{b}^{(0)}_k} =  \frac{1}{\sqrt{N}}  \bm{W}^{(1)}_k (\bm{W}^{(0)}_k)^2
    \dddot{\sigma}( \bm{W}^{(0)}_kx +  \bm{b}^{(0)}_k ).
\end{align*}
Then for $\bm{W}^{(0)}$, again we define
\begin{align*}
    &\bm{A}(t) = \left[ \bm{W}^{(1)} \right]^T \\
    & \bm{B}(t) = \bm{W}^{(0)} \\
    & \bm{C}(t) = \dddot{\sigma}(\bm{W}^{(0)} x + \bm{b}^{(0)} ) x \\
    & \bm{D}(t) = 2 \ddot{\sigma}(\bm{W}^{(0)} x + \bm{b}^{(0)}).
\end{align*}
Then,
\begin{align*}
     & \sup_{t \in [0, T]}  \left\|\frac{\partial u_{xx}(x, \bm{\theta}(t))}{\partial \bm{W}^{(0)}} - \frac{\partial u_{xx}(x, \bm{\theta}(0))}{\partial \bm{W}^{(0)}}      \right\|_2  \\
     &=\sup_{t \in [0, T]}  \left\|
     \frac{1}{\sqrt{N}} \Big( \bm{A}(t) \odot \bm{B}(t) \odot \left[ \bm{B}(t) \odot \bm{C}(t) + \bm{D}(t)  \right] -
     \bm{A}(0) \odot \bm{B}(0) \odot \left[ \bm{B}(0) \odot \bm{C}(0) + \bm{D}(0)  \right]         \Big)
     \right\|_2  \\
     &\leq  \sup_{t \in [0, T]}  \left\|
     \frac{1}{\sqrt{N}} \Big( \bm{A}(t) \odot \bm{B}(t) \odot  \bm{B}(t) \odot \bm{C}(t)  -
     \bm{A}(0) \odot \bm{B}(0) \odot \bm{B}(0) \odot \bm{C}(0)    \Big)
     \right\|_2 \\
     &+  \sup_{t \in [0, T]}  \left\|
     \frac{1}{\sqrt{N}} \Big( \bm{A}(t) \odot \bm{B}(t) \odot \bm{D}(t) -
     \bm{A}(0) \odot \bm{B}(0) \odot \bm{D}(0)          \Big)
     \right\|_2  \\
     &:= I_1 + I_2.
\end{align*}
For $I_1$, we have
\begin{align*}
    &\sup_{t \in [0, T]}  \left\|
     \frac{1}{\sqrt{N}} \Big( \bm{A}(t) \odot \bm{B}(t) \odot  \bm{B}(t) \odot \bm{C}(t)  -
     \bm{A}(0) \odot \bm{B}(0) \odot \bm{B}(0) \odot \bm{C}(0)    \Big)
     \right\|_2   \\
     & \leq \sup_{t \in [0, T]}  \left\|
     \frac{1}{\sqrt{N}} \Big( \big[\bm{A}(t) - \bm{A}(0) \big] \odot \bm{B}(t) \odot  \bm{B}(t) \odot \bm{C}(t) \Big)
     \right\|_2  \\
     &+ \sup_{t \in [0, T]}  \left\|
     \frac{1}{\sqrt{N}} \Big(\bm{A}(0) \odot \big[ \bm{B}(t) \odot  \bm{B}(t) \odot \bm{C}(t) - \bm{B}(0) \odot  \bm{B}(0) \odot \bm{C}(0) \big] \Big)
     \right\|_2 \\
     &\leq  \sup_{t \in [0, T]} \left\| \bm{B}(t)  \right\|_\infty^2  \left\| \bm{C}(t)  \right\|_\infty  \left\|  \frac{1}{\sqrt{N}} \Big( \bm{A}(t) - \bm{A}(0)  \Big) \right\|_2 \\
     &+\sup_{t \in [0, T]} \left\| \bm{A}(0)  \right\|_\infty  \left\|
     \frac{1}{\sqrt{N}} \Big( \bm{B}(t) \odot  \bm{B}(t) \odot \bm{C}(t) - \bm{B}(0) \odot  \bm{B}(0) \odot \bm{C}(0) \Big)
     \right\|_2 \\
     &\lesssim \sup_{t \in [0, T]}   \left\|  \frac{1}{\sqrt{N}} \Big( \bm{A}(t) - \bm{A}(0)  \Big) \right\|_2 +
     \sup_{t \in [0, T]} \left\|
     \frac{1}{\sqrt{N}} \Big( \bm{B}(t) \odot  \bm{B}(t) \odot \bm{C}(t) - \bm{B}(0) \odot  \bm{B}(0) \odot \bm{C}(0) \Big)
     \right\|_2 \\
     &\cdots \\
     &\lesssim    \sup_{t \in [0, T]}   \left\|  \frac{1}{\sqrt{N}} \Big( \bm{A}(t) - \bm{A}(0)  \Big) \right\|_2 +
     \sup_{t \in [0, T]}   \left\|  \frac{1}{\sqrt{N}} \Big( \bm{B}(t) - \bm{B}(0)  \Big) \right\|_2 \\
     &+
     \sup_{t \in [0, T]}   \left\|  \frac{1}{\sqrt{N}} \Big( \bm{C}(t) - \bm{C}(0)  \Big) \right\|_2 +
     \sup_{t \in [0, T]}   \left\|  \frac{1}{\sqrt{N}} \Big( \bm{D}(t) - \bm{D}(0)  \Big) \right\|_2 \\
     &\longrightarrow 0,
\end{align*}
as $N \rightarrow \infty$. We can use the same strategy to  $I_2$ as well as $\bm{W}^{(1)}$ and  $\bm{b}^{(0)}$. As a consequence, we conclude
\begin{align*}
     \lim_{N \rightarrow \infty} \sup_{t \in [0,T]}  \left\|   {\frac{\partial u_{xx}(x, \bm{\theta}(t))}{\partial \bm{\theta}}} -  {\frac{\partial u_{xx}(x, \bm{\theta}(0))}{\partial \bm{\theta}}}      \right\|_2 = 0.
\end{align*}
This concludes the proof.
\end{proof}

With these lemmas, now we can prove our main Theorem \ref{theorem: kernel_constant}

\begin{proof}[Proof of Theorem \ref{theorem: kernel_constant}]
For a given data set $\{x_b^i, g(x_b^i) \}_{i=1}^{N_b}, \{x_r^i, f(x_r^i) \}_{i=1}^{N_r}$, let $\bm{J}_u(t)$ and $\bm{J}_r(t)$ be the Jacobian matrix of $u(x_b, \bm{\theta}(t))$ and $u_{xx}(x_r, \bm{\theta})$ with respect to $\bm{\theta}$, respectively,
\begin{align*}
    \bm{J}_u(t) = \Big(  {\frac{\partial u(x_b^i, \bm{\theta}(t))}{\partial \bm{\theta}_j}} \Big),
     \bm{J}_r(t) = \Big(  {\frac{\partial u_{xx}(x_r^i, \bm{\theta}(t))}{\partial \bm{\theta}_j}} \Big).
\end{align*}
Note that
     \begin{align*}
          \bm{K}(t) = \begin{bmatrix}
                        \bm{J}_u(t) \\
                        \bm{J}_r(t)
                        \end{bmatrix}
                         \begin{bmatrix}
                        \bm{J}_u^T(t) , \bm{J}_r^T(t)
                        \end{bmatrix}
                          : = \bm{J}(t) \bm{J}^T(t).
     \end{align*}

This implies that
\begin{align*}
    \left\|     \bm{K}(t)   -   \bm{K}(0)  \right\|_2 &= \left\|     \bm{J}(t) \bm{J}^T(t)   -  \bm{J}(0) \bm{J}^T(0) \right\|_2 \\
    &\leq \left\|  \bm{J}(t) \big[\bm{J}^T(t)  -\bm{J}^T(0)    \big]             \right\|_2 +
    \left\|   \big[\bm{J}(t)  -\bm{J}(0)    \big]  \bm{J}^T(0)           \right\|_2 \\
    & \leq \left\|  \bm{J}(t)              \right\|_2 \left\|  \bm{J}(t)  -\bm{J}(0)      \right\|_2 + \left\|   \bm{J}(t)  -\bm{J}(0)  \right\|_2 \left\|   \bm{J}(0)   \right\|_2.
\end{align*}
By lemma \ref{lemma: output_grad_bounded}, it is easy to show that $\|\bm{J}(t)\|_2$ is bounded. So it now suffices to show that
\begin{align}
    & \sup_{t \in [0,T]} \left\|\bm{J}_u(t) - \bm{J}_u(0)   \right\|_F \rightarrow 0 \\
    &\sup_{t \in [0,T]}\left\|\bm{J}_r(t) - \bm{J}_r(0)   \right\|_F \rightarrow 0,
\end{align}
as $N \rightarrow \infty$. Since the training data is finite, it suffices to consider just two inputs $x,x'$. By the Cauchy-Schwartz inequality, we obtain
\begin{align*}
    &\left|  \left\langle  {\frac{\partial u(x, \bm{\theta}(t))}{\partial \bm{\theta}}} ,  {\frac{\partial u(x', \bm{\theta}(t))}{\partial \bm{\theta}}}  \right\rangle
    - \left\langle  {\frac{\partial u(x, \bm{\theta}(0))}{\partial \bm{\theta}}} ,  {\frac{\partial u(x', \bm{\theta}(0))}{\partial \bm{\theta}}}         \right\rangle     \right| \\
 & \leq     \left|  \left\langle  {\frac{\partial u(x, \bm{\theta}(t))}{\partial \bm{\theta}}} ,  {\frac{\partial u(x', \bm{\theta}(t))}{\partial \bm{\theta}}} -  {\frac{\partial u(x', \bm{\theta}(0))}{\partial \bm{\theta}}}  \right\rangle      \right|
 + \left|  \left\langle  {\frac{\partial u(x, \bm{\theta}(t))}{\partial \bm{\theta}}} - {\frac{\partial u(x, \bm{\theta}(0))}{\partial \bm{\theta}}}  ,  {\frac{\partial u(x', \bm{\theta}(0))}{\partial \bm{\theta}}}         \right\rangle     \right| \\
 & \leq \left\|  {\frac{\partial u(x, \bm{\theta}(t))}{\partial \bm{\theta}}}       \right\|_2  \left\|   {\frac{\partial u(x', \bm{\theta}(t))}{\partial \bm{\theta}}} -  {\frac{\partial u(x', \bm{\theta}(0))}{\partial \bm{\theta}}}      \right\|_2
 + \left\|   {\frac{\partial u(x, \bm{\theta}(t))}{\partial \bm{\theta}}} - {\frac{\partial u(x, \bm{\theta}(0))}{\partial \bm{\theta}}}      \right\|_2  \left\|   {\frac{\partial u(x', \bm{\theta}(0))}{\partial \bm{\theta}}}         \right\|_2.
\end{align*}
From Lemma \ref{lemma: output_grad_constant}, we have $\left\|  {\frac{\partial u(x, \bm{\theta}(t))}{\partial \bm{\theta}}}       \right\|_2$ is uniformly bounded for $t \in [0,T]$. Then using  Lemma \ref{lemma: output_grad_constant} again gives
\begin{align*}
      & \sup_{t \in [0,T]} \left|  \left\langle  {\frac{\partial u(x, \bm{\theta}(t))}{\partial \bm{\theta}}} ,  {\frac{\partial u(x', \bm{\theta}(t))}{\partial \bm{\theta}}}  \right\rangle
    - \left\langle  {\frac{\partial u(x, \bm{\theta}(0))}{\partial \bm{\theta}}} ,  {\frac{\partial u(x', \bm{\theta}(0))}{\partial \bm{\theta}}}         \right\rangle     \right| \\
    &\leq  C \sup_{t \in [0,T]}  \left\|   {\frac{\partial u(x', \bm{\theta}(t))}{\partial \bm{\theta}}} -  {\frac{\partial u(x', \bm{\theta}(0))}{\partial \bm{\theta}}}      \right\|_2 + C  \sup_{t \in [0,T]}  \left\|   {\frac{\partial u(x, \bm{\theta}(t))}{\partial \bm{\theta}}} - {\frac{\partial u(x, \bm{\theta}(0))}{\partial \bm{\theta}}}      \right\|_2 \\
    & \longrightarrow 0,
\end{align*}
as $N \rightarrow \infty$. This implies that
\begin{align*}
    \lim_{N \rightarrow \infty} \sup_{t \in [0,T]} \left\|\bm{J}_u(t) - \bm{J}_u(0)   \right\|_2 = 0.
\end{align*}
Similarly, we can repeat this calculation for $\bm{J}_{r}$, i.e.,
\begin{align*}
    & \sup_{t \in [0,T]} \left|  \left\langle  {\frac{\partial u_{xx}(x, \bm{\theta}(t))}{\partial \bm{\theta}}} ,  {\frac{\partial u_{xx}(x', \bm{\theta}(t))}{\partial \bm{\theta}}}  \right\rangle
    - \left\langle  {\frac{\partial u_{xx}(x, \bm{\theta}(0))}{\partial \bm{\theta}}} ,  {\frac{\partial u_{xx}(x', \bm{\theta}(0))}{\partial \bm{\theta}}}         \right\rangle     \right| \\
    &\leq  \sup_{t \in [0,T]} \left\|  {\frac{\partial u_{xx}(x, \bm{\theta}(t))}{\partial \bm{\theta}}}       \right\|_2  \left\|   {\frac{\partial u_{xx}(x', \bm{\theta}(t))}{\partial \bm{\theta}}} -  {\frac{\partial u_{xx}(x', \bm{\theta}(0))}{\partial \bm{\theta}}}      \right\|_2
 \\
 &+  \sup_{t \in [0,T]}\left\|   {\frac{\partial u_{xx}(x, \bm{\theta}(t))}{\partial \bm{\theta}}} - {\frac{\partial u_{xx}(x, \bm{\theta}(0))}{\partial \bm{\theta}}}      \right\|_2  \left\|   {\frac{\partial u_{xx}(x', \bm{\theta}(0))}{\partial \bm{\theta}}}         \right\|_2  \\
 & \leq C  \sup_{t \in [0,T]} \left\|   {\frac{\partial u_{xx}(x', \bm{\theta}(t))}{\partial \bm{\theta}}} -  {\frac{\partial u_{xx}(x', \bm{\theta}(0))}{\partial \bm{\theta}}}      \right\|_2 + C   \sup_{t \in [0,T]}\left\|   {\frac{\partial u_{xx}(x, \bm{\theta}(t))}{\partial \bm{\theta}}} - {\frac{\partial u_{xx}(x, \bm{\theta}(0))}{\partial \bm{\theta}}}      \right\|_2 \\
 &\longrightarrow 0,
\end{align*}
as $N \rightarrow \infty$. Hence, we get
\begin{align*}
     \lim_{N \rightarrow \infty} \sup_{t \in [0,T]} \left\|\bm{J}_r(t) - \bm{J}_r(0)   \right\|_2 = 0,
\end{align*}
and thus we conclude that
\begin{align*}
       \lim_{N \rightarrow \infty} \sup_{t \in [0,T]} \left\|     \bm{K}(t)   -   \bm{K}(0)  \right\|_2  =0.
\end{align*}
This concludes the proof.
\end{proof}

\end{document}